\documentclass[letterpaper]{article} 
\usepackage{aaai23}  
\usepackage{times}  
\usepackage{helvet}  
\usepackage{courier}  
\usepackage[hyphens]{url}  
\usepackage{graphicx} 
\usepackage{natbib}  
\usepackage{caption} 
\usepackage{algorithm}
\usepackage{algorithmic}

\usepackage{amsmath}
\usepackage{amsfonts}
\usepackage{amsthm}
\usepackage{amssymb}
\usepackage{mathtools}

\theoremstyle{plain}
\newtheorem{theorem}{Theorem}

\newtheorem{lemma}[theorem]{Lemma}
\newtheorem{corollary}[theorem]{Corollary}

\theoremstyle{definition}

\theoremstyle{remark}
\newtheorem{remark}[theorem]{Remark}

\newtheorem*{theorem*}{Theorem}
\newtheorem*{lemma*}{Lemma}
\newtheorem*{definition*}{Definition}
\newtheorem*{corollary*}{Corollary}
\newtheorem*{remark*}{Remark}

\DeclareMathOperator*{\E}{\mathbb{E}}
\newcommand{\s}{\mathcal{S}}
\newcommand{\A}{\mathcal{A}}
\newcommand{\M}{\mathcal{M}}
\newcommand{\T}{\mathcal{T}}

\def\vocabRWDchange{reward varying}{}
\def\vocabRWDchangeCAP{Reward Varying}{}

\def\vocabDYNchange{dynamics varying}{}
\def\vocabDYNchangeCAP{Dynamics Varying}{}

\def\vocabAppendix{Appendix}{} 

\def\vocabEq{Equation}{}

\def\vocabunreg{un-regularized}{}

\def\RewardChangeTheorem#1{
Let a task $\T$ with reward function $r$ be given, with the optimal value function $Q^*$ and corresponding optimal policy $\pi^*$. Consider a \vocabRWDchange{} task, $\widetilde{\T}$ with reward function $\widetilde{r}$, with an unknown optimal action-value function, $\widetilde{Q}^*$. Define $\kappa(s,a,s') \doteq \widetilde{r}(s,a, s') - r(s,a,s')$.\\
Denote the optimal action-value function $K^*$ as the solution of the following Bellman optimality equation
\begin{equation}\label{eq:K_backup1#1}
    K^{*}(s,a) = \E_{s' \sim{} p}\left[ \kappa(s,a,s') + \frac{\gamma}{\beta}  \log \E_{a' \sim{} \pi^*} e^{\beta K^{*}(s',a')} \right]
\end{equation}
and its corresponding state-value function
\begin{equation}
    V_K^*(s) = \frac{1}{\beta}\log \sum_a \pi^*(a\vert s)e^{\beta K^*(s,a)}
\end{equation}
Then, \begin{equation}\label{eq:Q=Q+K rwd_change#1}
    \widetilde{Q}^*(s,a) = Q^*(s,a) + K^*(s,a)
\end{equation}
\begin{equation}\label{eq:vtilda=v+vk for rwd change#1}
    \widetilde{V}^*(s) = V^*(s) + V_K^*(s)
\end{equation}
and 
\begin{equation}\label{eq:rwd change policies are the same#1}
    \widetilde{\pi}^*(a\vert s) = \pi^*_K(a\vert s)
\end{equation}
for all $s \in \s,\ a \in \A$.
}

\def\PolicyLemma#1{
Suppose the task $\T=\langle \s,\A,p,r,\gamma, \beta, \pi_0 \rangle$ is given with associated optimal value function $Q^*$ and optimal policy $\pi^*$. For a reward function
\begin{equation}\label{eq: change reward for prior policy adjustment#1}
    \widetilde{r}(s,a,s') = r(s,a,s') + \frac{1}{\beta}\log \frac{\pi_0(a\vert s)}{\pi_1(a\vert s)}
\end{equation}
the task $\widetilde{\T}=\langle \s,\A,p,\widetilde{r},\gamma, \beta, \pi_1 \rangle$
has optimal value functions
\begin{equation}
    \widetilde{Q}^*(s,a) = Q^*(s,a) + \frac{1}{\beta} \log \frac{\pi_0(a\vert s)}{\pi_1(a\vert s)}
\end{equation}
\begin{equation}
    \widetilde{V}^*(s) = V^*(s)
\end{equation}
and the optimal policies of $\T$ and $\widetilde{\T}$ are equal:
\begin{equation}
    \pi^*(a\vert s) = \widetilde{\pi}^*(a\vert s)
\end{equation}
for all $s \in \s, a \in \A$. 
}

\def\RewardShapingCorollary#1{
Let \vocabRWDchange{} tasks $\T$ and $\widetilde{\T}$ be given with corresponding solutions $(Q^*,V^*, \pi^*)$ and $(\widetilde{Q}^*, \widetilde{V}^*,\widetilde{\pi}^*)$.
The optimal value function for another \vocabRWDchange{} task, $\bar{\T}$ with the reward function 
\begin{align}\label{eq:rwd_shaping#1}
    \bar{r}(s,a,s')
    &= \widetilde{r}(s,a,s') + \gamma V^*(s') - V^*(s)
\end{align}
is given by
\begin{align}
    \bar{Q}^*(s,a) &= \widetilde{Q}^*(s,a) - Q^*(s,a) + \frac{1}{\beta}\log \frac{\pi^*(a \vert s)}{\pi_0(a \vert s)}\\
    &= \widetilde{Q}^*(s,a) - V^*(s)
\end{align}
and the corresponding optimal policy for $\bar{\T}$ is 
\begin{equation}
    \bar{\pi}^*(a \vert s) = \widetilde{\pi}^*(a \vert s).
\end{equation}}

\def\CaoLemma#1{
For a fixed policy $\pi(a\vert s) > 0$, discount factor $\gamma \in [0, 1)$, and an arbitrary choice of function $v \colon \s \to \mathbb{R}$,  there is a unique corresponding reward function

\begin{equation}\label{eq:cao thm 1#1}
    R(s,a,s') = \frac{1}{\beta}\log\frac{\pi(a\vert s)}{\pi_0(a\vert s)} +v(s) - \gamma v(s') 
\end{equation}
such that the task with reward $R$ yields an optimal value function $V^*(s) = v(s)$ and optimal policy $\pi^* = \pi$.
}

\def\PBRSTheorem#1{
Given task $\T = \langle \s,\A,p,r,\gamma, \beta, \pi_0 \rangle$ with optimal policy $\pi^*$ in entropy-regularized RL, then the \vocabRWDchange{} task $\T'$ 
with reward function
\begin{equation}\label{eq:eq for rewards in thm rwd shaping#1}
    r'(s,a,s')=r(s,a,s') + \gamma \Phi(s') - \Phi(s) 
\end{equation}
has the optimal policy $\widetilde{\pi}^* = \pi^*$, and its optimal value functions satisfy
\begin{equation}\label{eq:in thm rwd shaping q = q - phi#1}
    \widetilde{Q}^*(s,a) = Q^*(s,a) - \Phi(s)
\end{equation}
\begin{equation}\label{eq:in thm rwd shaping v = v - phi#1}
    \widetilde{V}^*(s) = V^*(s) - \Phi(s)
\end{equation}
for a bounded, but otherwise arbitrary function $\Phi \colon \s \to \mathbb{R}$.
}

\def\DynamicsChangeTheorem#1{
Let a task $\T = \langle \s,\A,p,r,\gamma, \beta, \pi_0 \rangle$ be given, with optimal value function $V^*$ and corresponding optimal policy $\pi^*$. Let a \vocabDYNchange{} task, $\widetilde{\T}=\langle \s,\A,q,r,\gamma, \beta, \pi_0 \rangle$, be given, with an unknown optimal action-value function $\widetilde{Q}^*(s,a)$. Assume that $r = r(s,a)$.
Denote the optimal action-value function $K^*$ as the solution of the following Bellman optimality equation
\begin{equation}\label{eq:K_p_backup2#1}
    K^{*}(s,a) = \kappa(s,a) + \frac{\gamma}{\beta} \E_{s' \sim{} q} \log \E_{a' \sim{} \pi^*} e^{\beta K^{*}(s',a')}
\end{equation}
where $\kappa$ is the corresponding reward function:
\begin{equation}
    \kappa(s,a)= \gamma \E_{s' \sim{} q} V^*(s') - \gamma \E_{s' \sim{} p} V^*(s').
\end{equation}
where $V^*(s)$ is the optimal state value function for the task defined by dynamics $p(s'\vert s,a)$.

Then, \begin{equation}\label{eq:dynamics Q=Q+K#1}
    \widetilde{Q}^*(s,a) = Q^*(s,a) + K^*(s,a)
\end{equation}
\begin{equation}\label{eq:dynamics V=V+K#1}
    \widetilde{V}^*(s) = V^*(s) + V_K^*(s)
\end{equation}
and 
\begin{equation}\label{eq:dynamics pi=piK#1}
    \widetilde{\pi}^*(a\vert s) = \pi^*_K(a\vert s)
\end{equation}
for all $s \in \s,\ a \in \A$.
}

\def\FreeSolutionsTheorem#1{
Consider the task $\T = \langle \s,\A,p,r,\gamma, \beta, \pi_0 \rangle$ with corresponding optimal value functions $Q^*(s,a), V^*(s)$ and optimal policy $\pi^*$. Then for a reward function
\begin{equation}\label{eq:rwd for free soln dynamics#1}
    \bar{r}(s,a) = r(s,a) - \gamma \E_{s' \sim{} q} V^*(s') + \gamma \E_{s' \sim{} p} V^*(s') 
\end{equation}
the task given by $\bar{\T} = \langle \s,\A,q,\bar{r},\gamma, \beta, \pi_0 \rangle$
has the same optimal action-value function, hence $\bar{V}^*=V^*$ and $\bar{\pi}^*=\pi^*$.
}

\def\CompositionTheorem#1{
Given a set of \vocabRWDchange{} tasks $\{\T^{(m)}\}$ with corresponding optimal value functions $\{Q^{(m)}\}$, denote $\widetilde{Q}^*$ as the optimal action-value function for the composition of $\{\T^{(m)}\}$ under $f$.
Define the value function $K^*$ as the solution of the following Bellman optimality equation

\small
\begin{equation}\label{eq:K in composition#1}
    K^*(s,a) = \E_{s' \sim{} p} \left[ \kappa(s,a,s') + \frac{\gamma}{\beta}  \log \E_{a' \sim{} \pi_f} e^{\beta K^*(s',a')}\right]
\end{equation}
\normalsize
where $\kappa$ is the corresponding reward function:
\begin{equation}\label{eq:rwd for composition#1}
    \kappa(s,a,s') = \left[ f(\{r^{(m)}\}) + \gamma V_f(s') \right] - f(\{Q^{(m)}(s,a)\}) 
\end{equation}
with the definition
\begin{equation*}
    V_f(s) = \frac{1}{\beta}\log\E_{a\sim{} \pi_0}\exp{\beta f(\{Q^{(m)}(s, a) \})},
\end{equation*}
and $\pi_f$ is the policy derived from $f(\{Q^{(m)}\})$:

\begin{equation}
    \pi_f(a\vert s) = \frac{\pi_0(a\vert s) e^{\beta f(\{Q^{(m)}(s,a)\})}}{e^{\beta V_f(s)}}
\end{equation}
Then, 
\begin{equation}\label{eq:Q=fQ+K#1}
    \widetilde{Q}^*(s,a) = f(\{Q^{(m)}(s,a)\}) + K^*(s,a)
\end{equation}
\begin{equation}
    \widetilde{V}^*(s) = V_f(s) + V_K^*(s)
\end{equation}
and 
\begin{equation}
    \widetilde{\pi}^*(a\vert s) = \pi^*_K(a\vert s)
\end{equation}
for all $s \in \s,\ a \in \A$.
}

\def\RobustRemark#1{
Given a potential-based reward shaping function $F$ as described in Theorem \ref{thm:rwd_shaping_ng}, the relations
\begin{equation}\label{eq:rwd shaping for non-optimal pi on q#1}
    \widetilde{Q}^{\pi}(s,a) = Q^{\pi}(s,a) - \Phi(s)
\end{equation}
\begin{equation}\label{eq:rwd shaping for non-optimal pi on v#1}
    \widetilde{V}^{\pi}(s) = V^{\pi}(s) - \Phi(s)
\end{equation}
also hold for non-optimal policies $\pi$. In particular, if $\pi$ is an $\epsilon$-optimal policy in the reward-shaped task (i.e. $||\widetilde{V}^\pi~-~\widetilde{V}^*||~<~\epsilon$) then, $\pi$ is also $\epsilon$-optimal in the original task (i.e. $||V^\pi-V^*||<\epsilon$).
}

\usepackage{newfloat}
\usepackage{listings}
\DeclareCaptionStyle{ruled}{labelfont=normalfont,labelsep=colon,strut=off} 
\floatstyle{ruled}
\newfloat{listing}{tb}{lst}{}
\floatname{listing}{Listing}
\title{Utilizing Prior Solutions for Reward Shaping and Composition in Entropy-Regularized Reinforcement Learning}
\author{
    Jacob Adamczyk,\textsuperscript{\rm 1} Argenis Arriojas,\textsuperscript{\rm 1} Stas Tiomkin,\textsuperscript{\rm 2} Rahul V. Kulkarni\textsuperscript{\rm 1}
}
\affiliations{
    \textsuperscript{\rm 1}Department of Physics, University of Massachusetts Boston\\
    \textsuperscript{\rm 2}Department of Computer Engineering, San Jose State University

    jacob.adamczyk001@umb.edu, arriojasmaldonado001@umb.edu, stas.tiomkin@sjsu.edu, rahul.kulkarni@umb.edu
}

\begin{document}
\maketitle
\begin{abstract}
    In reinforcement learning (RL), the ability to utilize prior knowledge from previously solved tasks can allow agents to quickly solve new problems. In some cases, these new problems may be approximately solved by composing the solutions of previously solved primitive tasks (task composition). Otherwise, prior knowledge can be used to adjust the reward function for a new problem, in a way that leaves the optimal policy unchanged but enables quicker learning (reward shaping). In this work, we develop a general framework for reward shaping and task composition in entropy-regularized RL. To do so, we derive an exact relation connecting the optimal soft value functions for two entropy-regularized RL problems with different reward functions and dynamics. We show how the derived relation leads to a general result for reward shaping in entropy-regularized RL. We then generalize this approach to derive an exact relation connecting optimal value functions for the composition of multiple tasks in entropy-regularized RL. We validate these theoretical contributions with experiments showing that reward shaping and task composition lead to faster learning in various settings.
\end{abstract}

\section{Introduction}

Reinforcement learning (RL) is a widely-used approach for training artificial agents to acquire complex behaviors and to engage in long-term decision making \cite{suttonBook}. Despite its great successes for goal-oriented tasks (e.g. board games such as chess and Go \cite{silver2018general}), RL approaches do not fare as well when the tasks change and become more complex. The underlying problem is that RL algorithms are often incapable of effectively reusing previously-acquired knowledge; as a consequence RL agents typically start from scratch when faced with new tasks and require vast amounts of training experience to learn new solutions. Therefore, a key challenge in the field is the development of RL approaches and algorithms  which are able to leverage the solutions of previous tasks for quickly solving a wide variety of new tasks. Developing approaches that enable such ``transfer learning'' is one of the problems we wish to address in the current work, in the context of \textit{entropy-regularized} reinforcement learning \cite{ZiebartThesis}.

A promising approach for transfer learning is based on composing solutions for previously solved tasks to obtain solutions for new tasks. The ability to combine primitive skills to learn more complex behaviors can lead to an exponential increase in the number of new problems that an agent is able to solve \cite{boolean,boolean_stoch}. Correspondingly, there is significant interest in this idea of compositionality of tasks in RL. The observation that entropy-regularized RL provides robust solutions \cite{eysenbach} and simple approaches for composing previous solutions \cite{Haarnoja2018, vanNiekerk, peng_MCP} in specific situations has led to increased interest in this topic. Several exact results for compositionality in entropy-regularized RL have been obtained, however these are based on highly limiting assumptions on the differences between the primitive tasks. The development of more general approaches to compositionality in entropy-regularized RL is currently an important challenge in the field.

Another challenge often encountered by RL agents solving new tasks is the problem of sparse reward signals. For example, if an agent only gains a reward at the end of a long and otherwise non-rewarding trajectory, it may be difficult to learn the optimal policy in this case since the agent must be sufficiently ``far-sighted". The field of reward shaping; wherein rewards are changed in a way that leaves the optimal policy invariant, has been used to address this issue \cite{ng_shaping}. These efforts have primarily focused on the standard RL framework; to the best of our knowledge, the corresponding results for reward shaping in entropy-regularized RL have not yet  been derived. Another related open question that motivates this work is understanding how we can utilize previously obtained solutions to implement reward shaping in entropy-regularized RL.

To address these issues, we focus on the core problem of deriving relations between optimal value functions for two tasks in entropy-regularized RL. Considering the first task as solved and the second as a new unsolved task, the derived relation defines a third task whose optimal value function allows us to solve the new task while leveraging prior knowledge. We show that the optimal policy for the task of interest is the same as the third task's. This observation leads to the derivation of a general result for reward shaping in entropy-regularized RL.
Based on these observations connecting the optimal value functions of two tasks, we derive principled methods of approaching both task composition and reward shaping in entropy-regularized RL. In doing so, we also extend the results of \cite{Hunt} for arbitrary functional transformations of rewards, and show that the theory of potential-based reward shaping of \cite{ng_shaping} also applies in the entropy-regularized RL formulation. By using the derived connection between optimal value functions, we also determine how a solution can remain optimal under new dynamics. Moreover, our results motivate a methodology for using previously acquired skills to learn and shape new entropy-regularized RL tasks.

\section{Prior Work}
There is a significant body of literature studying the problem of reward shaping in standard (\vocabunreg{}) reinforcement learning. This field was initiated by \cite{ng_shaping}, whose work introduced the concept of \textit{potential-based reward shaping} (PBRS). It was shown that PBRS functions are necessary and sufficient to describe the set of reward functions which yield the same optimal policies. In addition, the authors show that shaping is \textit{robust} (in the sense that near-optimal policies are also invariant) and amenable to usage of prior knowledge. Therefore, solutions to previous tasks, expert knowledge, or even heuristics can be used to ``shape" an agent's reward function in order to make a particular RL task easier to solve. One goal of the present work is to extend the results of \cite{ng_shaping} to the domain of entropy-regularized RL.

Although there exist many forms of transfer learning in RL \cite{taylor_survey}, we shall focus on the case of concurrent skill composition by a single agent. In this work, composition refers to the combination of previous tasks through their reward functions. Task composition was introduced in the entropy-regularized setting by \cite{Todorov2009} in the entropy-regularized setting and was advanced further by \cite{Haarnoja2018}. Since composition combines previously solved tasks' reward functions in some specified functional form, it is natural to assume that those previous solutions might also be combined in the same way to obtain an approximate solution to the new task. Indeed, in standard (\vocabunreg{}) RL, it was shown by \cite{boolean} 
that this holds in the case of Boolean compositions. However, for these equalities to hold, there are strong restrictions on the reward functions for the previous tasks: they may only differ on the absorbing (also known as terminal) states. In our work, we shall consider reward functions which can vary globally (over all states and actions) in entropy-regularized RL.

Previous work has shown that, in entropy-regularized RL, applying the same transformation to the solutions of previously solved tasks leads to a useful first-order approximation, depending on the specific transformation. In \cite{Hunt} the authors have derived a specific correction function that can be learned and used to correct this first-order approximation to exactly solve the new task. In this work we shall extend this result (Theorem 3.2 of \cite{Hunt}) to arbitrary functions (rather than only convex linear combinations) of reward functions.

\section{Preliminaries}

We consider the Markov Decision Process (MDP) model to study the entropy-regularized reinforcement learning problem. The MDP, denoted $\M$, consists of a state space $\s$, action space $\A$, transition dynamics $p \colon \s \times \A \times \s \to [0,1]$, (bounded) reward function $r \colon \s \times \A \times \s \to \mathbb{R}$, discount factor $\gamma<1$ and inverse temperature $\beta > 0$. We represent the MDP as a tuple $\M = \langle \s,\A,p,r,\gamma \rangle$. The discount factor $\gamma \in (0,1)$ discounts future rewards and assures convergence of the accumulated reward for an infinitely long trajectory ($T \to \infty$).
In many instances we will also specify the particular \textit{prior policy} $\pi_0 \colon \s \times \A \to (0,1)$, a probability distribution over actions, specifying an initial exploration, data-collection, or behavior policy. We assume that the prior policy is absolutely continuous with respect to any trial policy $\pi$ which ensures that the Kullback-Liebler divergence in \vocabEq{} \eqref{eq:objective_function} is well-defined and bounded.
Although implicit in some cases, we always assume an MDP's reward function is bounded.

The entropy-regularized framework for reinforcement learning augments
the standard reward-maximization objective with an entropic regularization term,
relative to a reference policy $\pi_0$:
\begin{equation}\label{eq:objective_function}
    J(\pi)=\E_{p,\pi}
    \left[ \sum_{t=1}^{T} \gamma^{t-1} \left( r_t - \frac{1}{\beta}
        \log\left(\frac{\pi(a_t|s_t)}{\pi_0(a_t|s_t)} \right) \right) \right]
\end{equation}

This objective leads to optimal policies which remain partially exploratory (depending on the entropic term's weight, $\beta^{-1}$) and robust under perturbations to the rewards and dynamics \cite{eysenbach}. Therefore, entropy-regularized RL presents a useful method for applying reinforcement learning in real-world settings where the dynamics and reward functions may not be known with full precision.

\begin{definition*}[Entropy-Regularized Task]
    An entropy-regularized task (or simply \textbf{task}) is an MDP $\M$ together with an inverse temperature $\beta$ and a prior policy $\pi_0$. We denote a task by $\T = \M \cup \langle \beta, \pi_0\rangle = \langle \s,\A,p,r,\gamma, \beta, \pi_0 \rangle$.
\end{definition*}
In the following sections, we assume the existence of a previously solved task $\T$ (or set of tasks $\{\T\}$), representing the agent's primitive knowledge. The ``solution" to the task $\T$ refers to the optimal soft action-value function (or simply the optimal value function), $Q^*$, satisfying the soft Bellman optimality equation:
\begin{equation}\label{eq:Q_orig}
    Q^*(s,a) = \E_{s' \sim{} p} \left[ r(s,a,s') + \frac{\gamma}{\beta} \log \E_{a' \sim{} \pi_0} e^{ \beta Q^*(s',a')} \right] 
\end{equation}
which can be solved by iterating a Bellman backup equation until convergence:
\small
\begin{equation}\label{eq:bellman_backup}
    Q^{(N+1)}(s,a) = \E_{s' \sim{} p} \left[ r(s,a,s') +
        \frac{\gamma}{\beta} \log \E_{a' \sim{} \pi_0} e^{ \beta Q^{(N)}(s',a')} \right]
\end{equation}
\normalsize
where $Q^{(0)}(s,a)$ is an arbitrary initialization function. Since the soft Bellman operator $\mathcal{B}$ is a contraction \cite{pmlr-v70-haarnoja17a}, any bounded initialization function $Q^{(0)}$ will converge to the optimal value function: $\lim_{k \to \infty} \mathcal{B}^k Q^{(0)} = Q^*$. 
When discussing the ``solution" of the task $\T$ we also refer to the optimal policy derived from $Q^*$:
\begin{equation}
    \pi^*(a|s) = \frac{\pi_0(a|s)e^{\beta Q^*(s,a)}}{\sum_{a'} \pi_0(a'|s)e^{\beta Q^*(s,a')}}
\end{equation}
as well as the optimal state-value function:
\begin{equation}\label{eq:definition of v}
    V^*(s) = \frac{1}{\beta}\log \sum_a \pi_0(a|s)e^{\beta Q^*(s,a)}
\end{equation}

In the following sections we study a solution's dependence on the underlying task's characteristics (namely, its reward function and dynamics). We show that these considerations naturally lead to the subjects of reward shaping and compositionality.

\section{Change of Rewards}
We begin by supposing that an agent has a solution to a single task, $\T$. The agent is then asked to solve a new problem, where only the reward function has changed. Beyond simply solving a new problem, this change in rewards may be caused by an adversary, perturbation, or general transform in the same domain. For example, suppose we have learned to reach a goal state in a maze, but now we are tasked with moving to a new goal in the same maze; as in Figure \ref{fig:rwd_shape_maze}.

We formulate this problem by considering the two tasks $\T = \langle \s,\A,p,r,\gamma, \beta, \pi_0 \rangle$ and  $\widetilde{\T} = \langle \s,\A,p,\widetilde{r},\gamma, \beta, \pi_0 \rangle$ differing only on their reward functions. We take this opportunity to introduce the following definition.

\begin{definition*}[\vocabRWDchangeCAP{} Tasks]
    Consider a set of tasks $\{\T^{(k)}\}_{k=1}^{N}$. If the tasks only vary on their reward functions; that is, they are of the form $\T^{(k)} = \langle \s,\A,p,r^{(k)},\gamma, \beta, \pi_0 \rangle$, then we say the set of tasks $\{\T^{(k)}\}$ is \textbf{\vocabRWDchange{}}.
\end{definition*}
In other words, we restrict our attention to those tasks which share the same state and action spaces, transition dynamics, discount factor, temperature, and prior policy.

With these definitions in place, we first address the following question: Assuming tasks $\T$ and $\widetilde{\T}$ are \vocabRWDchange{}, how can we utilize the solution of $\T$ when solving $\widetilde{\T}$?
The answer to this question is provided by the following theorem.
(The proofs for all results are provided in the \vocabAppendix{}.)

\begin{theorem}\label{thm:rwd_change}
    \RewardChangeTheorem{}
\end{theorem}

Therefore, by directly incorporating the solution of $\T$ into $\widetilde{Q}^*$ we can instead learn an auxiliary function $K^*$ which itself happens to be an optimal action-value function.
We can now use the same soft $Q$-learning algorithms \cite{haarnoja_SAC} for learning this corrective value function ($K^*$) via \vocabEq{} \eqref{eq:K_backup1}. In doing so, we learn the desired optimal value function for the new task: $\widetilde{Q}^*$.

As discussed in \cite{Hunt}, it is also possible to learn a corrective value function strictly offline, by using data collected for a previous task, $\T$. In the \vocabAppendix{}, we show that it is indeed possible to learn $K^*$ using offline updates. In such a setup, the advantage is that one requires no additional samples of the environment (the previous experience can be used with appropriately re-labelled rewards).
While learning the value function $K^*$ of Theorem \ref{thm:rwd_change} we are implicitly solving \textit{two} tasks simultaneously: one task corresponding to reward function $\kappa$ with prior policy $\pi^*$, and another task with a reward function $\widetilde{r}$ and prior policy $\pi_0$. In the following section, we show that the former task can be mapped onto yet another task which also has a prior policy $\pi_0$.

\section{Reward Shaping}
In this section we explore the connection between the result of Theorem \ref{thm:rwd_change} and the field of potential-based reward shaping. \vocabEq{} \eqref{eq:rwd change policies are the same} implies that the optimal policies are the same for two distinct entropy-regularized RL problems.
This is exactly the desired outcome of a shaped reward: the task's optimal policy is invariant to a change in the task's reward function.

However, these two entropy-regularized RL problems use different prior policies and therefore the result is not immediately applicable to reward shaping. To correct for this, we introduce the following lemma which describes how a change in prior policy can be accounted for in the rewards of an entropy-regularized RL task.
\begin{lemma}\label{lem:policy_change}
    \PolicyLemma{}
\end{lemma}
Therefore, if the prior policy is changed ($\pi_0 \to \pi_1$) in an entropy-regularized RL task, we can appropriately adjust the reward function (as written in \vocabEq{} \eqref{eq: change reward for prior policy adjustment}) in order to retain an optimal solution.
Therefore, by solving a task $\T$ with prior policy $\pi_0$, we have also simultaneously solved all tasks with an arbitrary prior policy $\pi_1 > 0$ and corresponding reward functions $\widetilde{r}$ given in \vocabEq{} \eqref{eq: change reward for prior policy adjustment}.

Now by applying Lemma \ref{lem:policy_change} to Theorem \ref{thm:rwd_change}, we immediately have the following reward shaping result for entropy-regularized RL.

\begin{corollary}[Reward Shaping]\label{thm:rwd_shaping1}
    \RewardShapingCorollary{}
\end{corollary}
Therefore, a shift in the reward function by $\gamma V^*(s') - V^*(s)$ does not change the optimal policy, for any reward function $r$ with corresponding optimal soft value functions $Q^*$ and $V^*$.
However, Corollary \ref{thm:rwd_shaping1} does not yet resemble a \textit{general} reward shaping theorem, since it \textit{requires} a solution to a different task ($\T$) to provide the shaping function, and we do not know whether such reward functions fully characterize \textit{all} reward functions with the same optimal policy $\widetilde{\pi}^*$. To address these issues, we combine Corollary \ref{thm:rwd_shaping1} with the following lemma appearing in \cite{cao2021identifiability}. In doing so, we arrive at a potential-based reward shaping theorem for entropy-regularized RL.

\begin{lemma}\cite{cao2021identifiability}\label{lem:inv_rwd}
    \CaoLemma{}
\end{lemma}

Informally, this result states that one can always construct a reward function $R$ such that an arbitrary $\pi$ and $v$ are optimal in the given environment (i.e. with fixed $p, \gamma, \beta$). Furthermore, for a fixed $\pi$ and $v$ in the given environment, the reward function $R$ is uniquely defined up to a constant shift.

By applying Lemma \ref{lem:inv_rwd} in light of Theorem \ref{thm:rwd_shaping1}, we see that the $V^*$ in \vocabEq{} \eqref{eq:rwd_shaping} can in fact be \textit{any} (bounded) function, $v \colon \s \to \mathbb{R}$: it will always represent an optimal value function for a \vocabRWDchange{} task. In fact, \citeauthor{cao2021identifiability} have proven that \textit{all} reward functions having the same optimal policy must be of the form shown in \vocabEq{} \eqref{eq:cao thm 1}. Therefore, the set of reward functions of this type describe \textit{all} rewards with the same optimal policy.

This naturally leads us to the generalization of the primary result of \cite{ng_shaping} in the setting of entropy-regularized RL:

\begin{theorem}[Potential-Based Reward Shaping]\label{thm:rwd_shaping_ng}
    \PBRSTheorem{}
\end{theorem}
Furthermore, due to the aforementioned result of \cite{cao2021identifiability}, we note that we also have the following necessary and sufficient conditions:
\begin{itemize}
    \item Sufficiency: Adding a \textit{potential-based} function
          $F(s,a,s')~=~\gamma~\Phi(s')~-~\Phi(s)$ to a task's reward function leaves the task's optimal policy unchanged.
    \item Necessity: Any shaping function which leaves a task's optimal policy invariant must be of the form $F$ above.
\end{itemize}

\begin{remark}
    Interestingly, the method described above for arriving at Theorem \ref{thm:rwd_shaping_ng} provides a useful way of considering the possible ``degrees of freedom" in the reward function by contrasting the work of \cite{ng_shaping} and \cite{cao2021identifiability}. Specifically, \cite{cao2021identifiability} considered an arbitrary function $v(s)$ which uniquely (up to a constant shift) identifies the reward function for a given optimal policy, as noted above. To make the connection with our results, we can take this arbitrary function to be the value function $V_K^*(s)$, since we have $\pi_K^* = \widetilde{\pi}^*$. Because of \vocabEq{} \eqref{eq:vtilda=v+vk for rwd change}, we can equivalently take this choice of $v(s)$ to fix $V^*(s)$, which can be any arbitrary function $\Phi(s)$, as noted in Theorem \ref{thm:rwd_shaping_ng} above.
    In other words, instead of considering the degree of freedom to be the value function $V_K^*$, we can alternatively consider it to be the value function $V^*$. The degree of freedom identified in \cite{cao2021identifiability} corresponds to $V_K^*$, whereas the degree of freedom identified by \cite{ng_shaping} corresponds to $V^*$. The two choices are equivalent, given \vocabEq{} \eqref{eq:vtilda=v+vk for rwd change}.
\end{remark}

The utility of Theorem \ref{thm:rwd_shaping_ng}, as compared to Corollary \ref{thm:rwd_shaping1}, is that we now have a way of constructing shaping functions without requiring knowledge of optimal policies in advance. Lemma \ref{lem:inv_rwd} alone would not have allowed for this, as it requires an optimal policy $\pi$ as input for the reward function $R$, which is more useful in the study of inverse reinforcement learning and identifiability. Of course, if the optimal policy $\pi^*$ \textit{is} known, then one can use Lemma \ref{lem:inv_rwd} to construct the entire range of shaped rewards by iterating over all possible functions $v(s)$. Considering the space of all reward functions ($\mathbb{R}^{|\s||\A|}$) Lemma \ref{lem:inv_rwd} carves out a ``degenerate" subspace or ``class" of dimension $|\s|$, whose members are defined by tasks with the same optimal policy.

Analogous to Remark 1 of \cite{ng_shaping}, we also have the following result in entropy-regularized reinforcement learning, which provides robustness to the shaping function $F$.
\begin{remark}\label{rmk:non-optimal shaping}
    \RobustRemark{}
\end{remark}
This robustness is a useful property, as it allows near-optimal policies in the reward-shaped task to be readily interpreted as near-optimal policies in the original task of interest. More generally speaking, we can say that the action of reward shaping preserves ordering in the space of policies.

\begin{figure}
    \begin{center}
        \includegraphics[width=0.43\textwidth]{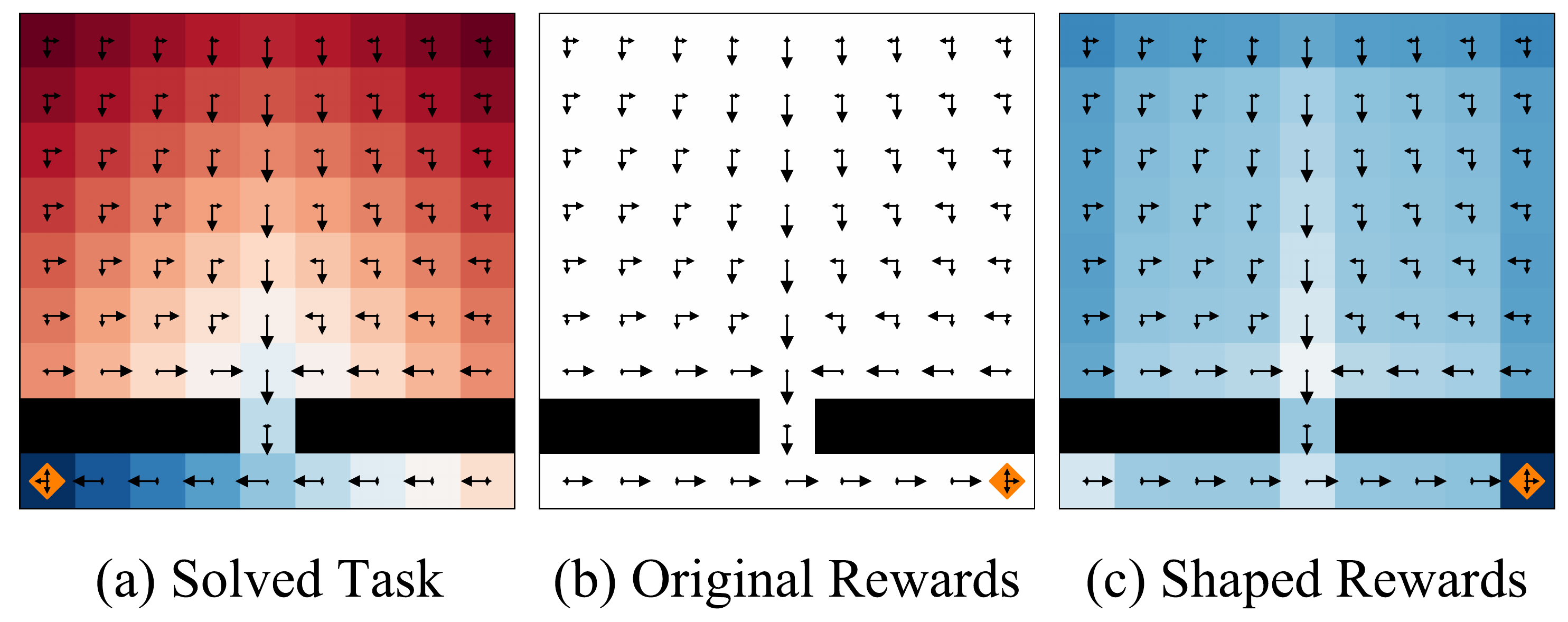}
        \caption{Demonstrating the reward shaping result of Corollary \ref{thm:rwd_shaping1}. Tasks in question are to navigate to the left or right corner in the bottom of this simple maze. The orange diamond represents the goal state for either task (blue represents regions of higher value). (a) Task with known solution ($\pi^*$ and $V^*(s)$ shown); (b) New task in the same environment ($\pi^*$ and $r$ shown); (c) Task defined in (b) with a shaped reward function (\vocabEq{}  \eqref{eq:rwd_shaping}), having the same optimal policy as the task with an unshaped reward function ($\pi^*$ and $r$ shown). In this experiment we use the parameters $\beta=3, \gamma =0.99$.} 
        \label{fig:rwd_shape_maze}
    \end{center}
\end{figure}

\begin{figure}
    \centering
    \includegraphics[width=0.4\textwidth]{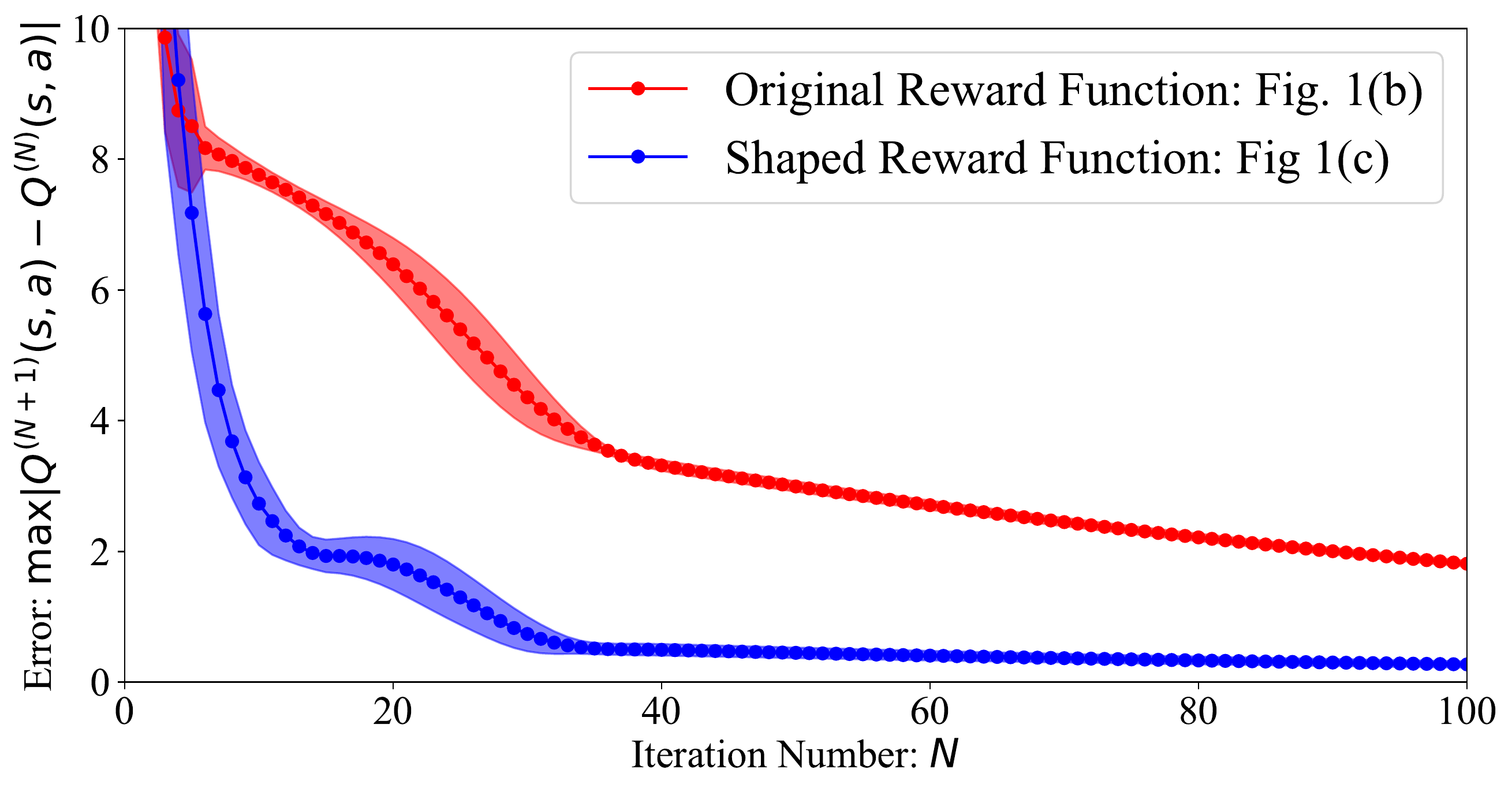}
    \caption{Convergence of the Bellman backup equation (\vocabEq{} \eqref{eq:bellman_backup}) for the unshaped and shaped task considered in Fig. \ref{fig:rwd_shape_maze}(b), \ref{fig:rwd_shape_maze}(c) respectively. Average taken over 10 random initializations, and one standard deviation is shown in the shaded region.}
    \label{fig:rwd_shape_converg}
\end{figure}

\subsection{Identifiability}
The problem of identifiability arises in the context of inverse reinforcement learning (IRL) where one observes an optimal policy and attempts to infer the underlying reward function. In \cite{cao2021identifiability}, it was argued that, under certain conditions, the underlying reward function $r$ is identifiable (up to a constant shift) when its corresponding optimal policy is observed in two sufficiently different environments with dynamics $p$ and $q$. As in \cite{cao2021identifiability}, we suppose the environments are further diversified by different discount factors $\gamma$ and $\widetilde{\gamma}$ respectively. In the following, we will see how our results provide insight into the conditions considered in \cite{cao2021identifiability} for identifiability using data from these two different environments.

Theorem \ref{thm:rwd_shaping_ng} can be used to derive the condition which makes it impossible to determine an underlying reward function $r$ given the optimal policies in $p$ and $q$. If it is possible to shape $r$ with a potential $\Phi(s)$ in $p$ and shape $r$ by another potential $\Psi(s)$ in $q$ in a way that makes the shaped reward functions identical, then identifiability is not possible. 
Using \vocabEq{} \eqref{eq:eq for rewards in thm rwd shaping} to equate two such shaped rewards, we arrive at the following condition:
\begin{equation}
    \gamma \E_{s' \sim{} p} \Phi(s) -  \Phi(s') = \widetilde{\gamma} \E_{s' \sim{} q} \Psi(s) -  \Psi(s')
\end{equation}
If the above condition is satisfied by non-trivial shaping potentials $\Phi$ and $\Psi$, then there are at least two reward functions (the unshaped and correspondingly shaped rewards) in dynamics $p$ and $q$ which are consistent with all the observed constraints, hence the reward function will not be identifiable. The condition that there are no such non-trivial shaping potentials is imposed by Definition 1 of \cite{cao2021identifiability} which then leads to their Theorem 2.

\section{Change of Dynamics}
Beyond having a new task whose sole distinction is in the reward function (\vocabRWDchange{} tasks), we can instead consider two tasks which differ in their transition dynamics:
\begin{definition*}[\vocabDYNchangeCAP{} Tasks]
    Consider a set of tasks $\{\T^{(k)}\}_{k=1}^{N}$. If the tasks only vary on their transition dynamics; that is, they are of the form $\T^{(k)} = \langle \s,\A,p^{(k)},r,\gamma, \beta, \pi_0 \rangle$ then we say the set of tasks $\{\T^{(k)}\}$ is \textbf{\vocabDYNchange{}}.
\end{definition*}

Consider a shift in the environment's dynamics $p(s'\vert s,a) \to q(s'\vert s,a)$ as represented by two \vocabDYNchange{} tasks. For example, in a discrete maze setting, the floor may become more slippery.
We again derive a corrective value function that utilizes the previous solution to a \vocabDYNchange{} task.
We now state the corresponding result for the corrective value function in this case:
\begin{theorem}\label{thm:dynamics_change}
    \DynamicsChangeTheorem{}
\end{theorem}

Therefore, in the face of changed dynamics, an agent can adapt by learning the value function $K^*$ instead of $\widetilde{Q}^*$. In this way, the agent uses the relevant knowledge already accumulated in a similar environment.

For simplicity, we have kept the rewards the same across the two tasks, but the results of Theorem \ref{thm:rwd_change} and Theorem \ref{thm:dynamics_change} can be readily combined to accommodate those tasks which have different dynamics \textit{and} different reward functions.

Examining \vocabEq{} \eqref{eq:dynamics Q=Q+K}, where $\widetilde{Q}^*$ and $K^*$ correspond to a task with dynamics $q$, and $Q^*$ corresponds to dynamics $p$, we can instead consider $\widetilde{Q}^*$ and $K^*$ as being related via \vocabEq{} \eqref{eq:Q=Q+K rwd_change} of Theorem \ref{thm:rwd_change}. With this perspective, $Q^*$ represents the optimal value-function of a \vocabRWDchange{} task with dynamics $q$ and a different reward function (denoted $\bar{r}$ below).
Therefore, given a solution to a task with dynamics $p$, we automatically have the solution to a task in dynamics $q$. The reward function $\bar{r}$ in a task with dynamics $q$ to which this optimal value function corresponds, is provided by the following theorem:

\begin{theorem}\label{thm:dynamics free soln}
    \FreeSolutionsTheorem{}
\end{theorem}

Interestingly, Theorem \ref{thm:dynamics free soln} implies that by solving problems in one environment with dynamics $p$, we are also simultaneously solving different problems in (arbitrary) other dynamics $q$. Hence, by learning in an environment that is ``safer" to experiment in ($p$), we can obtain solutions to tasks in another environment ($q$), perhaps where testing is more difficult, expensive, or dangerous. By assembling this set of rewards ($\bar{r}$) for which we have the solution in dynamics $q$, we can either (a) attempt to solve the inverse problem of finding a task to solve in $q$ such that that we solve the desired task in $p$ or (b) use the forthcoming results on compositionality and previous results on reward shifts to solve the task(s) of interest.

\begin{figure}
    \begin{center}
        \includegraphics[width=0.45\textwidth]{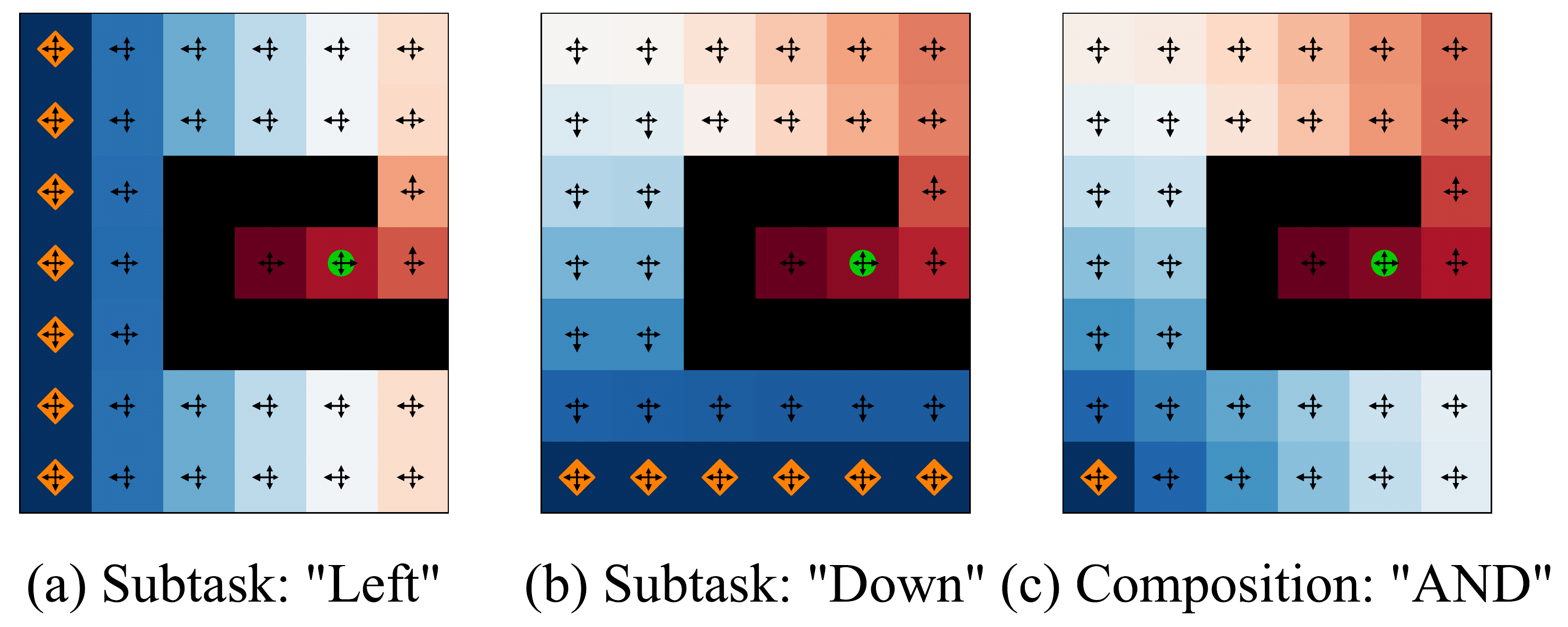}
        \caption{Composition of two subtasks (a) and (b), with function $f=\min(\cdot)$. All subfigures illustrate $\pi^*$ and $V^*(s)$. (a)-(b) Tasks with known solution (c) Composition of subtasks. In this experiment we use the parameters $\beta=2, \gamma =0.98$. The green circle indicates the agent's initial state and the orange diamond represents the goal state for either task.}
        \label{fig:comp_maze}
    \end{center}
\end{figure}

\begin{figure}[t]
    \centering
    \includegraphics[width=0.4\textwidth]{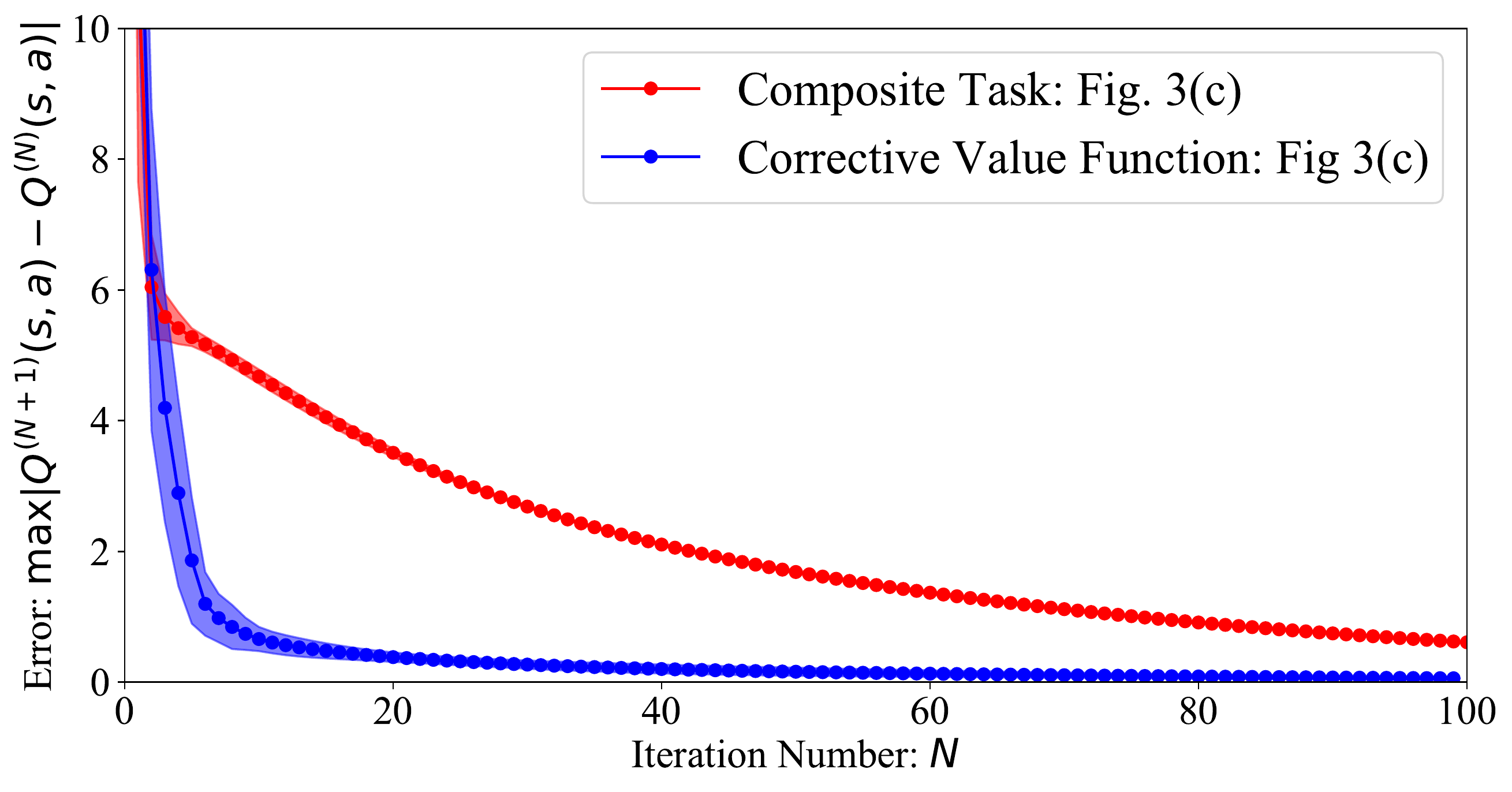}
    \caption{Convergence of the Bellman backup equation (\vocabEq{} \eqref{eq:bellman_backup}) for the task considered in Fig. \ref{fig:comp_maze}(c). Red is the composition ($\widetilde{Q}^*$) learned directly. Blue is the corrective value function ($K^*$). Average taken over 25 random initializations, and one standard deviation is shown in the shaded region.}
    \label{fig:comp_converg}
\end{figure}

\section{Composition of Rewards}
In this section we generalize the previous results by considering arbitrary compositions of $M$ \vocabRWDchange{} tasks. That is, we now consider a set of solved tasks $\{\T^{(m)}\}_{m=1}^{M}$ such that $\T^{(m)} = \langle \s,\A,p,r^{(m)},\gamma, \beta, \pi_0 \rangle$. To compose these tasks, we consider applying a function $f$ to the reward functions of $\{\T\}$. We also note the specific case of $M=1$ for transformations of a single task's reward function may be of interest.

\begin{definition*}[Task Composition]
    Consider a set of \vocabRWDchange{} tasks, $\{\T^{(m)}\}_{m=1}^{M}$. The \textbf{composition} of $\{\T^{(m)}\}_{m=1}^{M}$ under the (bounded) function $f \colon \mathbb{R}^M \to \mathbb{R}$ is defined as the mutually \vocabRWDchange{} task with reward function $r(s,a,s')=f(\{r^{(m)}(s,a,s')\})$.
\end{definition*}

Motivated by the results of \cite{Hunt}, we derive another corrective value function, which corrects the na\"ive guess of functionally transforming the value functions in the same way as the rewards (that is, $f(\{Q^{(m)}\})$). This na\"ive guess is in fact a bound in the case of convex combinations \cite{Haarnoja2018}. Learning the corrective value function for a simple composition task is illustrated in Figures \ref{fig:comp_maze} and \ref{fig:comp_converg}.

\begin{theorem}\label{thm:composition}
    \CompositionTheorem{}
\end{theorem}

The result of this theorem can be stated as follows: calculate a policy by transforming the optimal value functions in the same way the rewards were transformed. Using this new policy as the prior policy with an appropriate reward function ($\kappa$ defined in \vocabEq{} \eqref{eq:rwd for composition}), we can learn the correction term, $K^*$, to obtain the desired optimal value function, $\widetilde{Q}^*$. We again note that, as in Theorem \ref{thm:rwd_change}, it is possible to learn $K^*$ in an offline manner which is described in the proof of Theorem \ref{thm:composition} in the \vocabAppendix{}.

The fixed point $K^*$ in \vocabEq{} \eqref{eq:K in composition} generalizes the ``Divergence Correction'' ($C^\infty$) introduced by \cite{Hunt} for convex combinations of reward functions. Notice that the reward function $\kappa$ in Theorem \ref{thm:composition} measures the ``non-linearity'' of $f$. For if $f$ were linear (cf. Theorem 3.2 of \cite{Hunt}), then the first term (in brackets) cancels with the total transformed $Q$ function being subtracted, leaving the R\'enyi divergence between subtask policies.
In addition, we have also shown that $K^*$ is in fact the \textit{optimal} value function for a certain task: the task with rewards and prior policy as defined in Theorem \ref{thm:composition}.

\section{Discussion}
In this work, we have studied transfer learning in entropy-regularized reinforcement learning. Specifically, we have considered \vocabRWDchange{} tasks, \vocabDYNchange{} tasks, and composition of \vocabRWDchange{} tasks. By deriving a corrective value function in each case, we have shown that the solutions for new tasks can be informed by previous solutions. Interestingly, this study of corrective value functions also led to the derivation of potential-based reward shaping in entropy-regularized RL.

We have shown that optimal solutions under a given transition dynamics also corresponds to a set of optimal solutions under any other dynamics, by explicitly calculating the reward function in \vocabEq{} \eqref{eq:rwd for free soln dynamics}. This change in perspective between Theorem \ref{thm:rwd_change} and Theorem \ref{thm:dynamics_change} allows one to transfer a body of knowledge obtained in one dynamics $p$, to any other dynamics of interest. Although these are solutions to reward functions which may not be of interest \textit{a priori}, the solutions may still prove useful when used in tandem with the results of Theorem \ref{thm:rwd_change} and Theorem \ref{thm:composition}.

We have also generalized the ``Divergence Correction" result of \cite{Hunt}, allowing for general transformations and compositions over primitive tasks. All derived corrective value functions allow the agent to solve the task of interest by applying previous knowledge to the problem at hand.

\section{Limitations and Future Work}
Although the results and proofs are stated for discrete settings, it is straightforward to extend the results to continuous state and action spaces, with the usual assumptions for Bellman convergence. The tests here are demonstrated in discrete finite environments, but this work may also be extended to encompass continuous spaces.

This work is situated in the context of entropy-regularized RL, where the stochasticity of optimal policies allow for optimal solutions to be manipulated and combined for new tasks. Further work may explore the analogous problem in \vocabunreg{} RL, which can we understood as the limit $\beta \to \infty$.
We also note that the case of $\gamma = 1$
has straightforward proofs in the probabilistic inference framework \cite{LevineTutorial}. We intend to explore these results and their consequences in future work.

Just as we have derived corrective value functions for changes in reward function, prior policy (Lemma \ref{lem:policy_change}), and dynamics change, one might also consider tasks which differ in discount factor $\gamma$ or temperature $\beta^{-1}$. Although not stated here, similar results can be derived for these settings as well. These generalization and their consequences will be explored in future work.

Finally, we note that there may be more general results for a definition of composition which allow the dynamics of tasks to differ as well. This topic, and possible implications for sim-to-real and general transfer learning is currently being explored and will be left to future work.

\iftrue
    \appendix
    \section{Appendix Overview}
    In this Technical Appendix, we: (a) provide further experiments to expand on those in the Main Text; and (b) provide the proofs of all theoretical results presented in the Main Text.

    \section{Experiments}
    In this section, we explore the results of further experiments in the tabular domain, as introduced in the Main Text. These experiments serve to answer the following questions:
    \begin{itemize}
        \item Does reward shaping (using the derived results) improve training times?
        \item How is learning affected by the structure of the underlying tasks when using reward shaping?
    \end{itemize}
    We also provide a complementary (``OR" rather than ``AND") composition task, to be contrasted with Figures 3 and 4 of the Main Text.

    To address these items, we modify OpenAI's Frozen Lake environment (\cite{openAI}) to allow for model-based learning (wherein both dynamics and rewards are known).

    To calculate the optimal soft action-value functions, we use exact dynamic programming, by iterating the Bellman backup equation (\vocabEq{} \eqref{eq:bellman2}) in a model-based setting until convergence.
    \small
    \begin{equation}\label{eq:bellman2}
        Q^{(N+1)}(s,a) = \E_{s' \sim{} p} \left[ r(s,a,s') + \frac{\gamma}{\beta} \log \E_{a' \sim{} \pi_0} e^{ \beta Q^{(N)}(s',a')} \right] 
    \end{equation}
    \normalsize
    Furthermore, we calculate the errors during training; defined as the maximum difference between two successive iterations of Bellman backup:
    \begin{equation}
        \text{error} = \max_{(s,a)} |Q^{(k+1)}(s,a) - Q^{(k)}(s,a)|.
    \end{equation}

    \subsection{``OR'' Composition}
    In the Main Text (Figures 3 and 4) we demonstrated the results for an ``AND" composition, as defined in \cite{boolean}. The composition is taken over two subtasks in a spiral-shaped maze. Each subtask corresponds to reaching either the left or bottom wall of the maze. In this experiment we use an identical setup (same subtasks) but instead apply the ``OR'' composition (i.e. the composition function is now $f = \max( \cdot )$). Similar to the Main Text, we take $\beta=2$, $\gamma=0.98$.

    \begin{figure}[ht]
        \centering
        \includegraphics[width=0.43\textwidth]{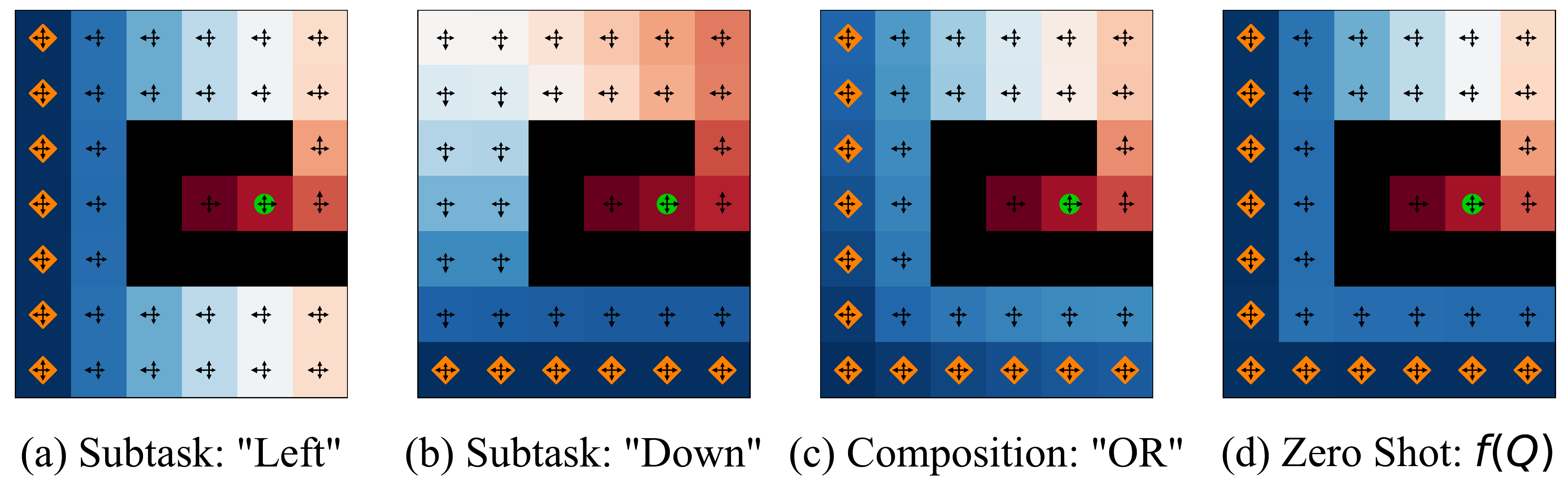}

        \caption{State value functions are illustrated in red-blue colormaps (with blue indicating regions of higher value). The orange diamonds represent rewarding states, definitive of each task. Corresponding policies are plotted with black arrows. (a)-(b): Subtasks, presumably with known solutions. (c): The solution to the composite problem. (d): The ``Zero Shot" approximation is the zeroth step of Theorem~\ref{thm:composition}, meaning $\widetilde{Q}^{(0)}(s,a)=\max\left(Q^{(\text{Left})}(s,a\right), Q^{(\text{Down})}(s,a))$.}
        \label{fig:ORcomp_maze}
    \end{figure}

    \begin{figure}[H]
        \centering
        \includegraphics[width=0.43\textwidth]{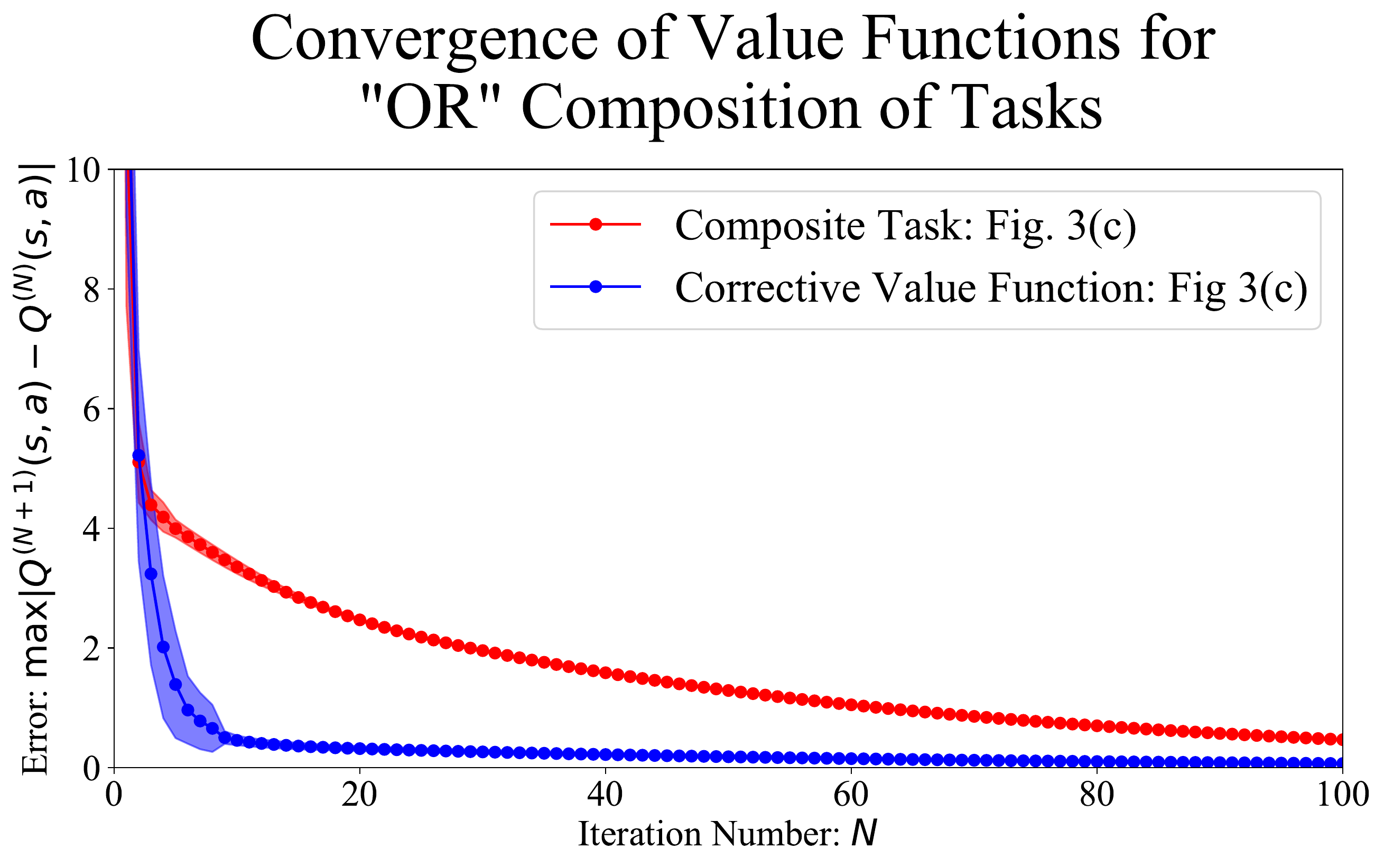}
        \caption{Convergence of the Bellman backup equation (\vocabEq{} \eqref{eq:bellman2}) for the task considered in Fig. \ref{fig:ORcomp_maze}. Red is the composition ($\widetilde{Q}^*$) learned directly. Blue is the corrective value function ($K^*$). Averages are taken over 25 random initializations, and one standard deviation is indicated by the shaded region.}
        \label{fig:ORcomp_converg}
    \end{figure}

    We observe that the corrective value function ($K^*$ defined in Theorem~\ref{thm:composition}) converges much faster than the composite task's value function ($\widetilde{Q}$). Furthermore, the zero shot approximation (Fig. \ref{fig:ORcomp_maze}(d)) is itself quite close to the optimal solution in this case. We conclude that learning $K^*$ is more advantageous than learning $\widetilde{Q}^*$, given the faster convergence time.

    \subsection{Dependence on Size of State Space}
    In the Main Text (Figures 1 and 2), we demonstrated our results on a simple maze with goal states in the bottom corners. We use one solved maze to shape the reward function for the other task (see Corollary \ref{thm:rwd_shaping1}). In this section and the next, we expand this experiment by considering two defining parameters: the size of the maze, and the location of the obstacle. Similar to the Main Text, we take $\beta=3$, $\gamma=0.99$.

    \begin{figure}[ht]
        \centering
        \includegraphics[width=0.43\textwidth]{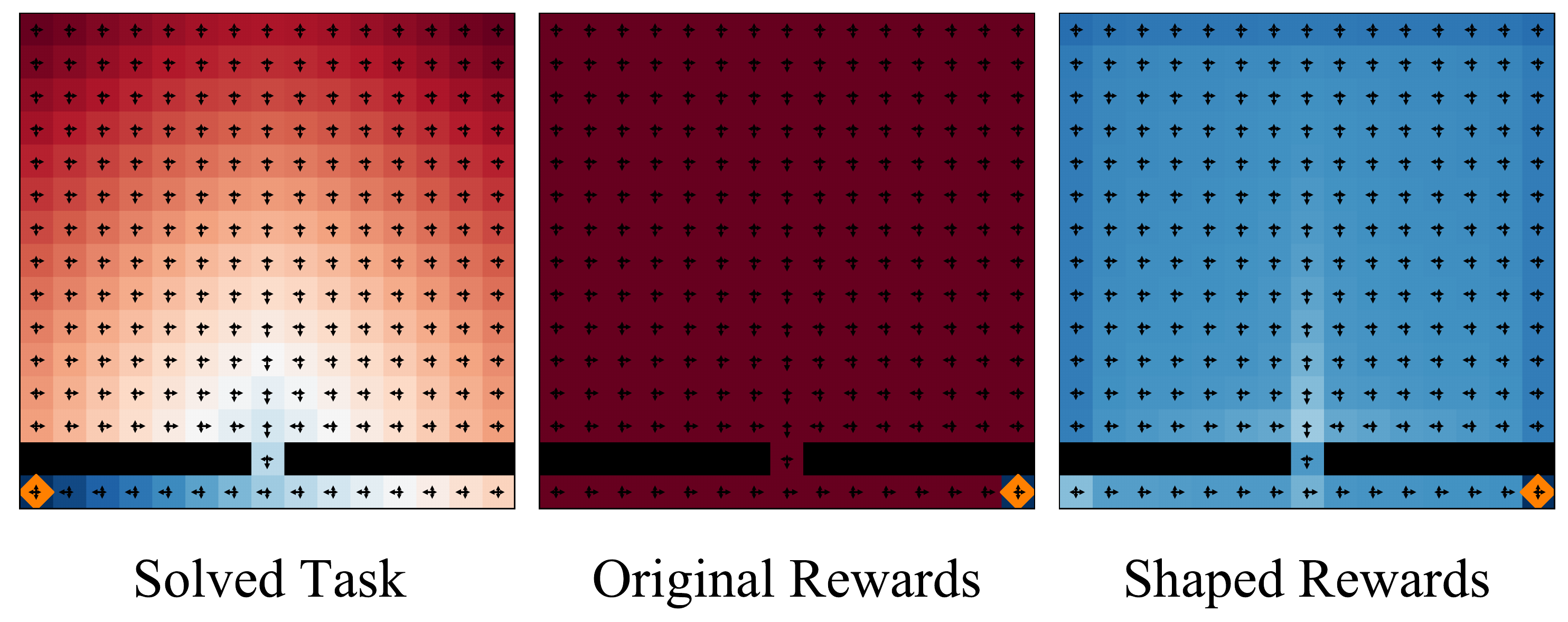}
        \caption{An example of a larger maze (15$\times$15), with a fixed wall height (always one unit above the goal states for all tasks).} 
        \label{fig:15x15maze}
    \end{figure}

    \begin{figure}[ht]
        \centering
        \includegraphics[width=0.43\textwidth]{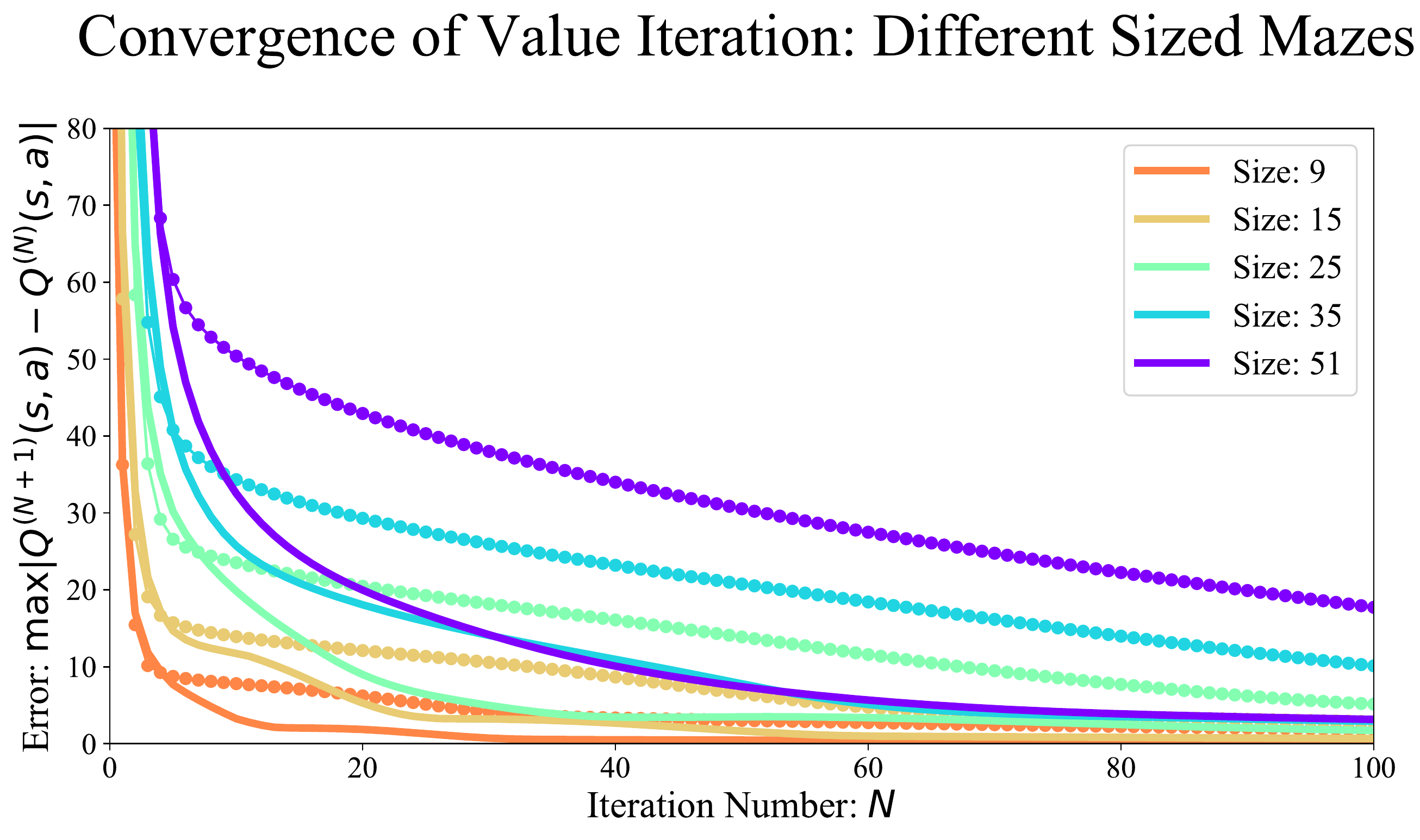}
        \caption{Convergence of the Bellman backup equation (\vocabEq{} \eqref{eq:bellman2}) for the tasks considered in Fig. \ref{fig:ORcomp_maze}. Solid lines indicate solutions with no reward shaping; lines with circles indicate solutions with reward shaping (by using the corresponding solved task, as can be seen in Fig. \ref{fig:15x15maze}). Average is taken over 5 samples (the standard deviation is omitted for clarity in the visualization).}
        \label{fig:goalchange_converg}
    \end{figure}

    As expected, the larger mazes' value function require more iterations until convergence. We observe that the task with shaped rewards converges faster than the task with the original sparse rewards in all cases. Furthermore, we note that larger mazes exhibit a more significant reduction in  training time (required number of iterations to reach some predefined error threshold).

    \subsection{Dependence on Wall Height}
    Similar to the previous section, we expand on the experiment of Figures 1 and 2 in the Main Text; in this experiment, the wall height is varied. Similar to the Main Text, we take $\beta=3$, $\gamma=0.99$.

    \begin{figure}[ht]
        \centering
        \includegraphics[width=0.43\textwidth]{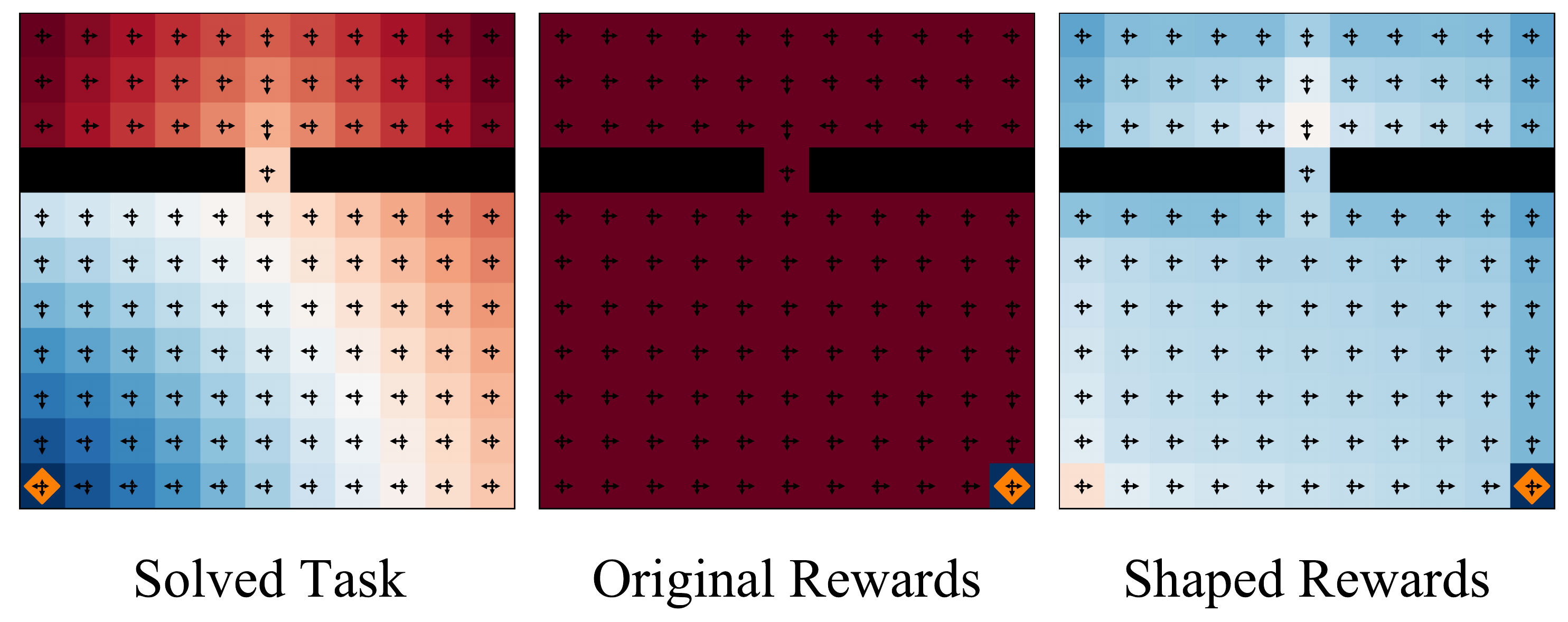}
        \caption{An example of a maze (11$\times$11), with a different wall height compared to the previously considered experiments. The ``wall height" is seven in this example: the wall is placed seven units above the goal states for all tasks.} 
        \label{fig:wallheight}
    \end{figure}

    \begin{figure}[ht]
        \centering
        \includegraphics[width=0.43\textwidth]{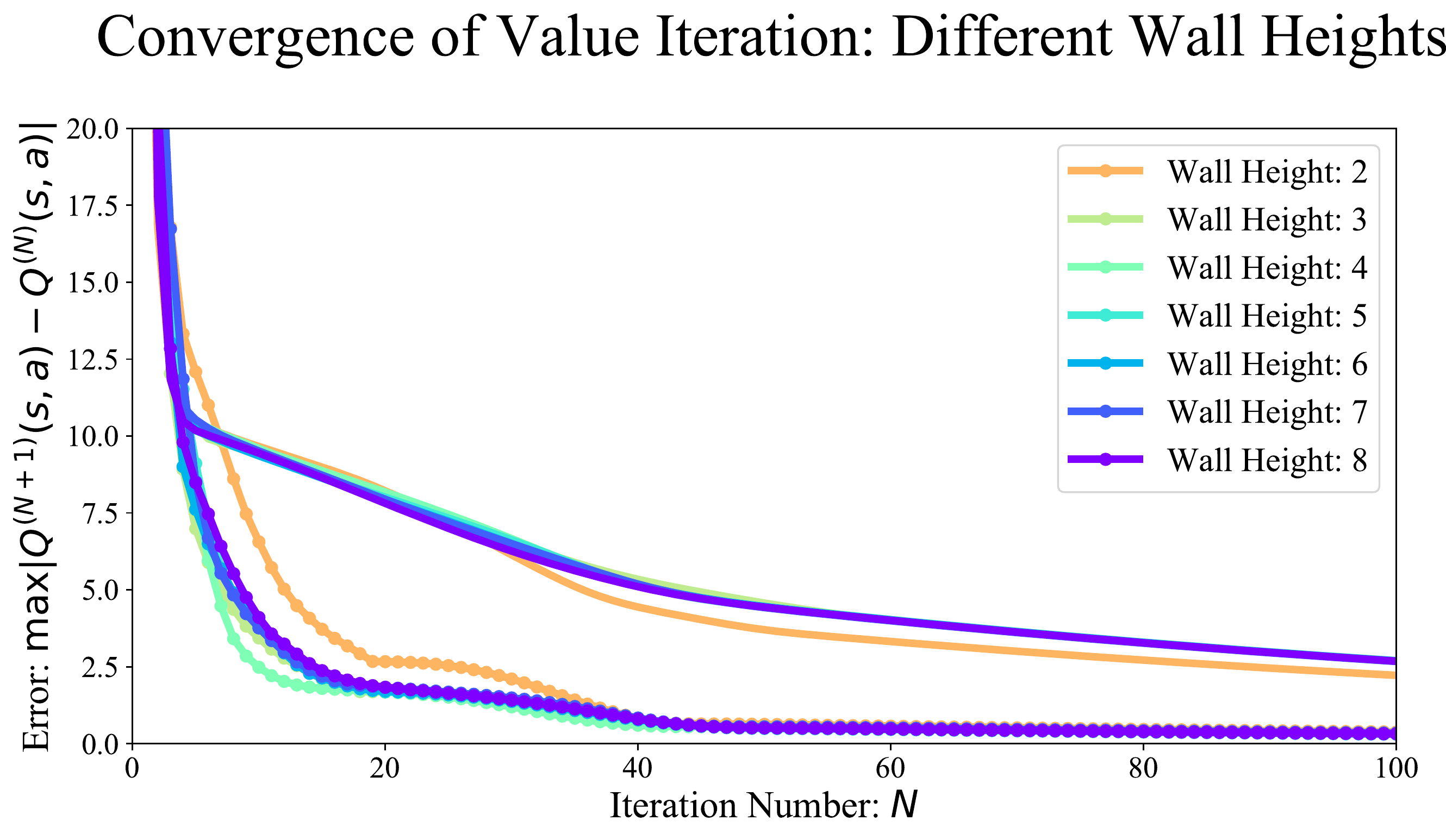}
        \caption{Convergence of the Bellman backup equation (\vocabEq{} \eqref{eq:bellman2}) for the tasks with varying wall height, such as considered in Fig. \ref{fig:wallheight}. Solid lines indicate solutions obtained without reward shaping; lines with circles indicate solutions obtained with reward shaping (by using the solved task as shown in Fig. \ref{fig:wallheight}). Average is taken over 5 samples (the standard deviation is omitted for clarity in the visualization).}
        \label{fig:wallheight_converg}
    \end{figure}

    In this experiment, we use ``wall height" as a proxy for measuring the amount of overlap between the solved task's and the task of interest's value functions. Interestingly, we observe that the convergence rates are approximately the same, independent of the wall height. The value function for the task with shaped rewards converges faster than the task with the original sparse rewards in all cases.

    \section{Proofs}
    We now prove the results presented in the Main Text. We restate the results here before their proof for convenience.
    \subsection{Proof of Theorem \ref{thm:rwd_change}}

    \begin{theorem*}
        \RewardChangeTheorem{app}
    \end{theorem*}
    \begin{proof}
        Let reward functions $r$ and $\widetilde{r} = r+\kappa$ be given, with corresponding optimal value functions $Q^*$ and $\widetilde{Q}^*$. First note the following useful identity for the optimal policy (corresponding to $Q^*$),
        \begin{equation*}
            \pi^*(a|s) = \frac{\pi_0(a|s) e^{\beta Q^*(s,a)}}{\sum_{a'}\pi_0(a'|s) e^{\beta Q^*(s,a')}}
        \end{equation*}
        or equivalently,
        \begin{equation}\label{eq:useful_relation_piQV}
            \frac{\pi^*(a|s)}{\pi_0(a|s)}e^{\beta V^*(s)} = e^{\beta Q^*(s,a)}.
        \end{equation}

        We will prove \vocabEq{} \eqref{eq:K_backup1} by induction, relating $\widetilde{Q}^{(N)}$, $Q^*$, and $K^{(N)}$ at each backup step.
        \begin{equation}\label{eq:rwd_KQ_relation}
            \widetilde{Q}^{(N)}(s,a) = Q^*(s,a) + K^{(N)}(s,a)
        \end{equation}

        Using the initialization $\widetilde{Q}^{(0)} = Q^*$, $K^{(0)}=0$, we write the iteration for the Bellman backup equation and insert the inductive assumption to obtain
        \begin{align*}
             & \widetilde{Q}^{(N+1)}(s,a) = \E_{s' \sim{} p} \biggl[ \widetilde{r}(s,a,s') + \\ &\frac{\gamma}{\beta} \log \E_{a' \sim{} \pi_0} \exp \left( \beta Q^*(s',a') + \beta K^{(N)}(s',a')\right) \biggr]
        \end{align*}
        Notice that ending here yields the offline result mentioned in the Main Text: the expectation over actions is over the prior policy $\pi_0$ (which is the prior policy for the task of interest, $\widetilde{\T}$). Therefore, one can use old trajectory data (tuples of $\{s,a,r,s',a'\}$) with appropriately re-labelled rewards $r \to \widetilde{r}$.

        Continuing instead, we can use \vocabEq{} \eqref{eq:useful_relation_piQV}, solving for $\exp \beta Q^*$ and re-writing the expectation over actions (notice the cancellation with $\pi_0$):
        \begin{align*}
             & \widetilde{Q}^{(N+1)}(s,a) = \E_{s' \sim{} p} \biggl[ r(s,a,s') + \gamma V^*(s') + \\ &\kappa(s,a,s') +  \frac{\gamma}{\beta} \E_{s' \sim{} p} \log \E_{a' \sim{} \pi^*} \exp \left( \beta K^{(N)}(s',a')\right) \biggr]
        \end{align*}
        Noting that we can simplify this by using the Bellman optimality equation for $Q^*$, we now have

        \begin{align*}
             & \widetilde{Q}^{(N+1)}(s,a) = Q^*(s,a) + \E_{s' \sim{} p} \biggl[ \kappa(s,a,s') + \\ &\frac{\gamma}{\beta} \E_{s' \sim{} p} \log \E_{a' \sim{} \pi^*}\exp \left( \beta K^{(N)}(s',a')\right) \biggr]
        \end{align*}

        \noindent In applying the backup equation for $K$,
        \begin{align*}\label{eq:K_backup2}
            K^{(N+1)}(s,a) & = \E_{s' \sim{} p} \biggl[ \kappa(s,a,s') + \\ &\frac{\gamma}{\beta} \E_{s' \sim{} p} \log \E_{a' \sim{} \pi^*} e^{\beta K^{(N)}(s',a')} \biggr]
        \end{align*}
        we see that
        \begin{equation}
            \widetilde{Q}^{(N+1)}(s,a) = Q^*(s,a) + K^{(N+1)}(s,a).
        \end{equation}
        Since the Bellman backup equation converges to the optimal value function from any bounded initialization:
        \begin{align*}
            \lim_{N \to \infty} \widetilde{Q}^{(N)} & = \widetilde{Q}^* \\
            \lim_{N \to \infty} K^{(N)}             & = K^*,
        \end{align*}
        this completes the proof of \vocabEq{} \eqref{eq:Q=Q+K rwd_change}. Note that since this is a relation regarding optimal value functions, it is valid for any initializations $\widetilde{Q}^{(0)}$ and $K^{(0)}$.

        We now prove \vocabEq{} \eqref{eq:vtilda=v+vk for rwd change}. By definition,
        \begin{equation}
            e^{\beta \widetilde{V}^*(s)} = \sum_a \pi_0(a|s)e^{\beta (Q^*(s,a) +  K^*(s,a))}
        \end{equation}
        Again using \vocabEq{} \eqref{eq:useful_relation_piQV} on $Q^*$ and substituting for $\pi_0(a \vert s) e^{\beta Q^*(s,a)}$, we have
        \begin{equation}
            e^{\beta \widetilde{V}^*(s)} = \sum_a \pi^*(a|s)e^{\beta V^*(s)} e^{\beta K^*(s,a)}
        \end{equation}
        By using the definition of $V^*_K(s)$ it immediately follows that
        $\widetilde{V}^*(s) = V^*(s) + V_K^*(s)$.
        Finally, we prove that $\widetilde{\pi}^*(a \vert s) = \pi_K^*(a \vert s)$ using the previous relations as follows:
        \begin{equation}
            \pi_K^*(a|s) = \frac{\pi^*(a|s) e^{\beta K^*(s,a)}}{e^{\beta V_K^*(s)}}
        \end{equation}
        \begin{equation}
            \pi_K^*(a|s) = \frac{\pi^*(a|s) e^{\beta \widetilde{Q}^*(s,a)} e^{-\beta Q^*(s,a)
                    }}{e^{\beta V_K^*(s)}}
        \end{equation}
        Using \vocabEq{} \eqref{eq:useful_relation_piQV} again to substitute for $\pi^*(a \vert s) e^{-\beta Q^*(s,a)}$:
        \begin{equation}
            \pi_K^*(a|s) = \frac{\pi_0(a|s) e^{\beta \widetilde{Q}^*(s,a)}}{e^{\beta (V_K^*(s)+ V^*(s))}}.
        \end{equation}
        Finally, using \vocabEq{} \eqref{eq:vtilda=v+vk for rwd change}, we find that indeed $\widetilde{\pi}^*(a \vert s) = \pi_K(a \vert s)$.
    \end{proof}

    \subsection{Proof of Lemma \ref{lem:policy_change}}
    \begin{lemma*}
        \PolicyLemma{app}
    \end{lemma*}
    \begin{proof}
        We again prove the result by induction, through successive iterations of the Bellman backup equation. By proving that
        \begin{equation}\label{eq:inductive equation}
            \widetilde{Q}^{(N)}(s,a) = Q^{(N)}(s,a) + \frac{1}{\beta} \log \frac{\pi_{0}(a \vert s)}{\pi_{1}(a \vert s)}
        \end{equation}
        holds for all $N$, then letting $N \to \infty$, we will have the desired result.
        We set $\widetilde{Q}^{(0)}=Q^{(0)}=0$. The base case $N=1$ is trivial, as it only considers the rewards defined above.
        We therefore begin with the inductive assumption, assuming \vocabEq{} \eqref{eq:inductive equation} holds for some $N>1$.
        We then wish to prove \vocabEq{} \eqref{eq:inductive equation} for the $(N+1)^{\text{th}}$ step of the backup equation. To this end, we write the corresponding backup equation for $\widetilde{Q}$ and apply the inductive assumption:
        \begin{align*}
            \widetilde{Q}^{(N+1)}(s,a) & =  \E_{s' \sim{} p} \biggl[ r(s,a,s') + \frac{1}{\beta} \log \frac{\pi_{0}(a \vert s)}{\pi_{1} (a \vert s)} + \\ &\frac{\gamma}{\beta} \log \E_{a' \sim{} \pi_{1}}  \exp \beta \widetilde{Q}^{(N)}(s',a') \biggr] \\
            \widetilde{Q}^{(N+1)}(s,a) & = \E_{s' \sim{} p} \biggl[ r(s,a,s') + \frac{1}{\beta} \log \frac{\pi_{0}(a \vert s)}{\pi_{1}(a \vert s) } +  \\ &\frac{\gamma}{\beta} \log \E_{a' \sim{} \pi_{1}} \frac{\pi_{0}(a' \vert s')}{\pi_{1}(a' \vert s')} \exp \beta Q^{(N)}(s',a') \biggr] \\
            \widetilde{Q}^{(N+1)}(s,a) & = \E_{s' \sim{} p} \biggl[ r(s,a,s') + \frac{1}{\beta} \log \frac{\pi_{0}(a \vert s)}{\pi_{1}(a \vert s) }  + \\ &\frac{\gamma}{\beta} \log \E_{a' \sim{} \pi_0} \exp \beta Q^{(N)}(s',a') \biggr] \\
        \end{align*}
        Recognizing the latter term (together with $r$) is simply $Q^{(N+1)}(s,a)$, we have:
        \begin{equation}
            \widetilde{Q}^{(N+1)}(s,a) = Q^{(N+1)}(s,a) + \frac{1}{\beta} \log \frac{\pi_{0}(a \vert s)}{\pi_{1}(a \vert s)}
        \end{equation}
        By taking the limit $N \to \infty$, we have the desired result. To prove $\widetilde{\pi}^* = \pi^*$, we simply write out their respective definitions and apply the previously obtained result.
    \end{proof}

    \subsection{Proof of Corollary \ref{thm:rwd_shaping1}}

    \begin{corollary*}
        \RewardShapingCorollary{app}
    \end{corollary*}
    \begin{proof}
        This result is a direct application of Lemma \ref{lem:policy_change} to Theorem \ref{thm:rwd_change}.

    \end{proof}

    \subsection{Proof of Lemma \ref{lem:inv_rwd}}
    \begin{lemma*}\cite{cao2021identifiability}
        \CaoLemma{app}
    \end{lemma*}
    \begin{proof}
        We note the only distinction between our statement and the original result (Theorem 1 of \cite{cao2021identifiability}) is the dependence on $s'$. However, by defining $R$ as we have above, their proof must only be amended by keeping track of $s'$ and the relevant expectation over next states.
    \end{proof}

    \subsection{Proof of Theorem \ref{thm:rwd_shaping_ng}}
    \begin{theorem*}
        \PBRSTheorem{app}
    \end{theorem*}
    We can now prove Theorem \ref{thm:rwd_shaping_ng} in light of the previous results:
    \begin{proof}
        Let $\Phi(s)$ be an arbitrary (bounded) function. Consider the reward varying tasks $\tau=\langle \s,\A,p,\rho,\gamma, \beta, \pi_0 \rangle$  and $\widetilde{\T}=\langle \s,\A,p,\widetilde{r},\gamma, \beta, \pi_0 \rangle$, where
        \begin{equation}
            \rho(s,a,s') = \Phi(s) - \gamma \Phi(s').
        \end{equation}
        Notice that by construction, we have taken the task $\tau$ to have optimal value function $\Phi(s)$ and optimal policy $\pi_0$ for simplicity. This is guaranteed by Lemma \ref{lem:inv_rwd}.

        Using Theorem \ref{thm:rwd_change}, we define the corrective-value function's task \mbox{$\T_K=\langle \s,\A,p,\kappa=\widetilde{r}-\rho,\gamma,\beta, \pi_\rho=\pi_0\rangle$}.

        By applying Lemma \ref{lem:policy_change} to $\T_K$ we can re-write this task as $\T_K = \langle \s,\A,p,\kappa',\gamma, \beta, \pi_0 \rangle$
        where
        \begin{equation}
            \kappa'(s,a,s') = \widetilde{r}(s,a,s') - \rho(s,a,s')
        \end{equation}

        We can re-write the above expression for $\kappa$ as:
        \begin{equation}
            \kappa(s,a,s') = \widetilde{r}(s,a,s') + \gamma \Phi^*(s') - \Phi^*(s)
        \end{equation}
        Due to the result of Theorem \ref{thm:rwd_change}, the optimal policies corresponding to $\widetilde{\T}$ and $\T_K$ are the same, $\widetilde{\pi}^*=\pi_K^*$, and the calculations leading to \vocabEq{}s  \eqref{eq:in thm rwd shaping q = q - phi} and \eqref{eq:in thm rwd shaping v = v - phi} are straightforward, using Theorem \ref{thm:rwd_change}. Following \cite{ng_shaping}, we call $F(s,a,s')=\gamma \Phi(s') - \Phi(s)$ the \textit{potential-based reward shaping function}.

    \end{proof}

    \subsection{Proof of Remark \ref{rmk:non-optimal shaping}}
    \begin{remark*}
        \RobustRemark{app}
    \end{remark*}
    \begin{proof}
        We will prove the result by using induction to show that the equality holds throughout the soft policy evaluation (\cite{pmlr-v80-haarnoja18b}):
        \begin{equation}\label{eq:remark proof}
            \widetilde{Q}^{(k)}(s,a) = Q^{(k)}(s,a) - \Phi(s)
        \end{equation}
        We use $k$ to distinguish that the previous equation is for policy iteration as opposed to the Bellman backup iteration used in other proofs. We subsequently take the limit $k\to \infty$ to obtain the result of Remark~\ref{rmk:non-optimal shaping}.

        We omit the $\pi$ superscript to ease the notation.
        First, we define $\widetilde{Q}^{(0)}(s,a) \doteq -\Phi(s)$, $Q^{(0)}\doteq 0$. We note that \vocabEq{} \eqref{eq:remark proof} is trivially satisfied for $k=0$, and begin by using the inductive assumption: assuming \vocabEq{} \eqref{eq:remark proof} holds for some $k>0$. Then,

        \begin{align*}
            \widetilde{Q}^{(k+1)}(s,a) & = \E_{s' \sim{} p} \biggl[ r(s,a,s') + \\ &\gamma \Phi(s') - \Phi(s) + \\ &\gamma \E_{a' \sim{} \pi} \left[\widetilde{Q}^{(k)}(s',a') - \frac{1}{\beta}\log\frac{\pi(a'|s')}{\pi_0(a'|s')} \right] \biggr]
        \end{align*}
        Using the inductive assumption in the above equation, we find
        \begin{align*}
             & \widetilde{Q}^{(k+1)}(s,a) = \E_{s' \sim{} p} \biggl[ r(s,a,s') - \Phi(s) + \gamma \Phi(s') + \\ &\gamma \E_{a' \sim{} \pi} \left[Q^{(k)}(s',a') -\Phi(s') - \frac{1}{\beta}\log\frac{\pi(a'|s')}{\pi_0(a'|s')} \right] \biggr]
        \end{align*}
        which simplifies to
        \begin{align*}
             & \widetilde{Q}^{(k+1)}(s,a) = \E_{s' \sim{} p} \biggl[ r(s,a,s') - \Phi(s) + \\ &\gamma \E_{a' \sim{} \pi} \left[Q^{(k)}(s',a') - \frac{1}{\beta}\log\frac{\pi(a'|s')}{\pi_0(a'|s')} \right] \biggr].
        \end{align*}
        Now, by recognizing the soft policy evaluation of $Q^{(k)}$, we have
        \begin{align*}
            \widetilde{Q}^{(k+1)}(s,a) & = Q^{(k+1)}(s,a) - \Phi(s)
        \end{align*}
        The limit $k\to \infty$ yields the desired result.
        To derive \vocabEq{} \eqref{eq:rwd shaping for non-optimal pi on v}, simply apply the definition of the state-value function given in \vocabEq{} \eqref{eq:definition of v} of the Main Text.

        For the $\epsilon$-optimality of $V^\pi$, take
        \begin{equation}
            ||\widetilde{V}^\pi - \widetilde{V}^*|| < \epsilon
        \end{equation}
        and substitute \vocabEq{} \eqref{eq:rwd shaping for non-optimal pi on v} and \vocabEq{} \eqref{eq:in thm rwd shaping v = v - phi} to obtain the desired result,
        \begin{equation}
            ||V^\pi - V^*|| < \epsilon.
        \end{equation}
        We note that we also have the stronger result that $\widetilde{V}^\pi(s) - \widetilde{V}^*(s) = V^\pi(s) - V^*(s)$ for all $s \in \s$.
    \end{proof}

    \subsection{Proof of Theorem \ref{thm:dynamics_change}}
    \begin{theorem*}
        \DynamicsChangeTheorem{app}
    \end{theorem*}

    \begin{proof}
        Again, we prove by induction
        \begin{equation}
            \widetilde{Q}^{(N)}(s,a) = Q^*(s,a) + K^{(N)}(s,a)
        \end{equation}
        with the understanding that $Q^*$ (resp. $\widetilde{Q}^*$) is the optimal value function under dynamics $p$ (resp. $q$). Starting in the same way as the previous proofs (instantiating with the base case and applying the inductive assumption), we have
        \begin{align*}
             & \widetilde{Q}^{(N+1)}(s,a) =  r(s,a) + \\ &\E_{s' \sim{} q } \biggl[\frac{\gamma}{\beta} \log \E_{a' \sim{} \pi_0} \exp \left( \beta Q^*(s',a') + \beta K^{(N)}(s',a')\right) \biggr]
        \end{align*}
        similar as before this leads to
        \begin{align*}
             & \widetilde{Q}^{(N+1)}(s,a) = r(s,a) + \\ &\E_{s' \sim{} q} \biggl[\frac{\gamma}{\beta} \left[ \beta V^*(s') + \log \E_{a' \sim{} \pi^*} \exp \left(\beta K^{(N)}(s',a')\right) \right] \biggr]
        \end{align*}
        Here we split up the expectation over $V^*$ as an expectation over the old dynamics $p$ (for which $V^*$ is optimal) and the new dynamics $q$. This allows us to recombine $r$ with $V^*$ (using $Q^*(s,a) = r(s,a) + \gamma \E_{s' \sim{} p} V^*(s')$):
        \begin{align*}
            \widetilde{Q}^{(N+1)}(s,a) & = Q^*(s,a) + \gamma \left( \E_{s' \sim{} q} - \E_{s' \sim{} p} \right) V^*(s') + \\ &\frac{\gamma}{\beta} \E_{s' \sim{} q} \log \E_{a' \sim{} \pi^*} \exp \left(\beta K^{(N)}(s',a')\right)
        \end{align*}

        This allows us to again read off the backup \vocabEq{} for the corrective value function ($K$) as:
        \begin{align*}\label{eq:K_p_backup}
            K^{(N+1)}(s,a) & = \kappa(s,a) + \\ &\frac{\gamma}{\beta} \E_{s' \sim{} q} \log \E_{a' \sim{} \pi^*}\exp \left( \beta K^{(N)}(s',a')\right)
        \end{align*}
        with rewards $\kappa$ given by
        \begin{equation}
            \kappa(s,a) = \gamma \E_{s' \sim{} q} V^*(s') - \gamma \E_{s' \sim{} p} V^*(s').
        \end{equation}
        Equations \eqref{eq:dynamics V=V+K} and \eqref{eq:dynamics pi=piK} follow from their definitions, similar to the proof of Theorem \ref{thm:rwd_change}.
    \end{proof}

    \subsection{Proof of Theorem \ref{thm:dynamics free soln}}
    \begin{theorem*}
        \FreeSolutionsTheorem{app}
    \end{theorem*}

    \begin{proof}
        Consider the task $\T=\langle \s,\A,p,r,\gamma, \beta, \pi_0 \rangle$ with optimal action-value function $Q^*$, and consider an intermediate \vocabDYNchange{} task $\widetilde{\T}=\langle \s,\A,q,r,\gamma, \beta, \pi_0 \rangle$ with optimal action-value function $\widetilde{Q}^*$. Together $Q^*$ and $\widetilde{Q}^*$ obey Theorem \ref{thm:dynamics_change}:
        \begin{equation}\label{eq:thm11_pf1}
            \widetilde{Q}^*(s,a) = Q^*(s,a) + K^*(s,a)
        \end{equation}
        where $K^*$ is the corresponding corrective value function which allows for the change in dynamics $p \to q$. From an alternate perspective, we shall again consider the same $\widetilde{Q}^*$ and $K^*$ (notably, their corresponding task definitions have the same dynamics $q$) in light of Theorem \ref{thm:rwd_change}:
        \begin{equation}\label{eq:thm11_pf2}
            \widetilde{Q}^*(s,a) = \bar{Q}^*(s,a) + K^*(s,a)
        \end{equation}
        where $\bar{Q}$ is the optimal action value function for a task $\bar{\T} = \langle \s,\A,q,\bar{r},\gamma, \beta, \pi_0 \rangle$ with rewards given by
        \begin{align*}
            \bar{r}(s,a) & = \widetilde{r}(s,a) - \kappa(s,a)                                            \\
                         & = r(s,a) - \gamma \E_{s' \sim{} q} V^*(s') + \gamma \E_{s' \sim{} p} V^*(s').
        \end{align*}
        Since $\widetilde{Q}^*$ and $K^*$ are the same in \vocabEq{} \eqref{eq:thm11_pf1} and \eqref{eq:thm11_pf2}, it immediately follows that the two tasks have the same optimal $Q$-function: $\bar{Q}^*(s,a) = Q^*(s,a)$.
    \end{proof}

    \subsection{Proof of Theorem \ref{thm:composition}}
    \begin{theorem*}
        \CompositionTheorem{app}
    \end{theorem*}
    \begin{proof}
        Similar to the previous proofs, we will prove the following equation by induction
        \begin{equation}\label{eq:inductive_pf}
            \widetilde{Q}^{(N)}(s,a) = f(Q(s,a)) + K^{(N)}(s,a)
        \end{equation}
        (notice that we shall drop brackets and indices around the $Q$ and $r$ task-wise sets for brevity). 
        Writing the Bellman backup for $\widetilde{Q}$ and inserting the inductive assumption,

        \begin{align*}
             & \widetilde{Q}^{(N+1)}(s,a) = \E_{s' \sim{} p} \biggl[ f(r(s,a,s')) + \\ &\frac{\gamma}{\beta} \log \E_{a'\sim{} \pi_0} \exp \beta \left( f(Q(s',a'))+K^{(N)}(s',a')\right) \biggr]
        \end{align*}
        Again, by stopping here we obtain the offline result.
        Next we insert a state-value function, $V_f$, as defined in the theorem statement, so as to introduce the policy $\pi_f$, also defined in the theorem statement:
        \begin{align*}
            \widetilde{Q}^{(N+1)}(s,a) & = \E_{s' \sim{} p} \biggl[ f(r(s,a,s')) + \gamma \E_{s' \sim{} p} V_f(s') + \\ &\frac{\gamma}{\beta} \log \E_{a'\sim{} \pi_f} \exp \beta K^{(N)}(s',a') \biggr]
        \end{align*}

        \noindent In applying the backup equation for $K$,

        \begin{align*}
            K^{(N+1)}(s,a) & =  \E_{s' \sim{} p} \biggl[ \kappa(s,a,s') + \\ &\frac{\gamma}{\beta} \log \E_{a' \sim{} \pi_f}\exp \left( \beta K^{(N)}(s',a')\right) \biggr]
        \end{align*}

        \noindent we see that

        \begin{equation}\label{eq:Q=fQ+K in proof}
            \widetilde{Q}^{(N+1)}(s,a) = f(\{Q^*(s,a)\}) + K^{(N+1)}(s,a).
        \end{equation}
        Since the Bellman backup equation converges to the optimal value function: $\lim_{N \to \infty} \widetilde{Q}^{(N)} = \widetilde{Q}^*$ and $\lim_{N \to \infty} K^{(N)}~=~K^*$, this completes the proof of \vocabEq{} \eqref{eq:Q=fQ+K}. From exponentiation of \vocabEq{} \eqref{eq:Q=fQ+K} we have
        \begin{equation}
            e^{\beta \widetilde{V}^*(s)} = \sum_a \pi_0(a|s) e^{\beta f(\{Q^*(s,a)\})} e^{\beta K^*(s,a)}
        \end{equation}\label{eq:vtilda in comp proof}
        Proceeding in a similar fashion to the proof of Theorem \ref{thm:rwd_change}, we now use the relation
        \begin{equation}
            \frac{\pi_f(a|s)}{\pi_0(a|s)} = \frac{e^{\beta f(\{Q^*(s,a)\})}}{e^{\beta V_f(s)}}
        \end{equation}
        and the definitions of $\pi_f$ and $V_f$ in the preceding equation for $\widetilde{V}^*$, which gives
        \begin{equation}
            e^{\beta \widetilde{V}^*(s)} = e^{\beta V_f(s)} \sum_a \pi_f(a|s)  e^{\beta K^*(s,a)}
        \end{equation}
        Since $\pi_f$ corresponds to the prior policy for the task with optimal value function $K^*$, this leads immediately to
        \begin{equation}
            \widetilde{V}^*(s) = V_f(s) + V_K^*(s).
        \end{equation}
        We omit the proof of $\widetilde{\pi}^*=\pi_K^*$, which is similar to the analogous result in Theorem~\ref{thm:rwd_change}.
    \end{proof}
\fi


\section*{Acknowledgments}
The authors would like to thank the anonymous reviewers for their helpful comments and suggestions.
This work was supported by the National Science Foundation through Award DMS-1854350,
the Proposal Development Grant provided by the University of Massachusetts Boston,
the Research Foundation at San Jose State University, and the Alliance Innovation Lab in Silicon Valley.
\bibliography{aaai23}

\end{document}